\documentclass[nohyperref]{article}

\usepackage{microtype}
\usepackage{graphicx}
\usepackage{subfigure}
\usepackage{booktabs} 

\usepackage{hyperref}
\usepackage{multirow}

\usepackage[accepted]{icml2022}

\usepackage{amsmath}
\usepackage{amssymb}
\usepackage{mathtools}
\usepackage{amsthm}

\usepackage[capitalize,noabbrev]{cleveref}

\theoremstyle{plain}
\newtheorem{theorem}{Theorem}[section]

\newtheorem{lemma}[theorem]{Lemma}
\newtheorem{corollary}[theorem]{Corollary}
\theoremstyle{definition}
\newtheorem{definition}[theorem]{Definition}

\theoremstyle{remark}

\usepackage{amsmath}
\usepackage{amsthm}
\usepackage{xcolor}
\usepackage{algorithm, algorithmic}
\usepackage{makecell}

\usepackage[textsize=tiny]{todonotes}

\icmltitlerunning{Towards a general theory for scalable masked Transformers}

\begin{document}

\twocolumn[
\icmltitle{From block-Toeplitz matrices to differential equations on graphs: towards a general theory for scalable masked Transformers}



\icmlsetsymbol{equal}{*}

\begin{icmlauthorlist}
\icmlauthor{Krzysztof Choromanski}{equal,gbr,cu}
\icmlauthor{Han Lin}{equal,cu}
\icmlauthor{Haoxian Chen}{equal,cu}
\icmlauthor{Tianyi Zhang}{cu}
\icmlauthor{Arijit Sehanobish}{ir}
\icmlauthor{Valerii Likhosherstov}{uc}
\icmlauthor{Jack Parker-Holder}{uo}
\icmlauthor{Tamas Sarlos}{gr}
\icmlauthor{Adrian Weller}{uc,ati}
\icmlauthor{Thomas Weingarten}{g}
\end{icmlauthorlist}

\icmlaffiliation{gbr}{Google Brain Robotics}
\icmlaffiliation{cu}{Columbia University}
\icmlaffiliation{ir}{Independent Researcher}
\icmlaffiliation{uc}{University of Cambridge}
\icmlaffiliation{uo}{University of Oxford}
\icmlaffiliation{gr}{Google Research}
\icmlaffiliation{ati}{The Alan Turing Institute}
\icmlaffiliation{g}{Google}

\icmlcorrespondingauthor{Krzysztof Choromanski}{kchoro@google.com}

\icmlkeywords{Machine Learning, ICML}

\vskip 0.3in
]



\printAffiliationsAndNotice{\icmlEqualContribution} 

\begin{abstract}
In this paper we provide, to the best of our knowledge, the first comprehensive approach for incorporating various masking mechanisms into Transformers architectures in a scalable way. We show that recent results on linear causal attention \cite{choromanski} and log-linear RPE-attention \cite{rpe-performers} are special cases of this general mechanism. However by casting the problem as a topological (graph-based) modulation of unmasked attention, we obtain several results unknown before, including efficient $d$-dimensional RPE-masking and graph-kernel masking. We leverage many mathematical techniques ranging from spectral analysis through dynamic programming and random walks to new algorithms for solving Markov processes on graphs. We provide a corresponding empirical evaluation.     
\end{abstract}

\vspace{-5mm}
\section{Introduction \& Related Work}
\label{sec:intro_related_work}

Transformers \cite{vaswani, gpt3, devlin} have revolutionized machine learning by reintroducing an \textit{attention mechanism} explicitly modeling complicated relationships between elementary ingredients of the ML models' inputs, e.g. words for text data, or patches/pixels for the image data \cite{vit-survey, dosovitskiy}. Crucially, attention quantifies these relationships via dynamic weights that depend on the input data. This architectural solution is the strength and at the same time the weakness of Transformer models.
An attention matrix scales quadratically in the length of the input sequence, making corresponding computations prohibitively expensive for longer inputs. 

Several solutions were proposed to address this limitation.
Local attention \cite{local_attention, parmar2} explicitly narrows down the attention context to a fixed-size window, effectively zeroing out most attention weights. In applications where long-range attention is crucial (e.g. protein modeling), other techniques were introduced. These include: (1) pooling mechanisms compressing sequences to shorter-ones agglomerating multiple-tokens signal \cite{avsec, funnel}, (2) hashing/clustering methods sparsifying attention by giving up attention modeling for tokens from different learnable hash-buckets/clusters \cite{reformer, routing}, (3) low-rank/kernel-based methods decomposing the attention matrix \cite{masked,choromanski, elu-perf, rfatt, nystromformer} and other \cite{cosformer}.

\begin{figure*}[h]
    \includegraphics[width=.99\linewidth]{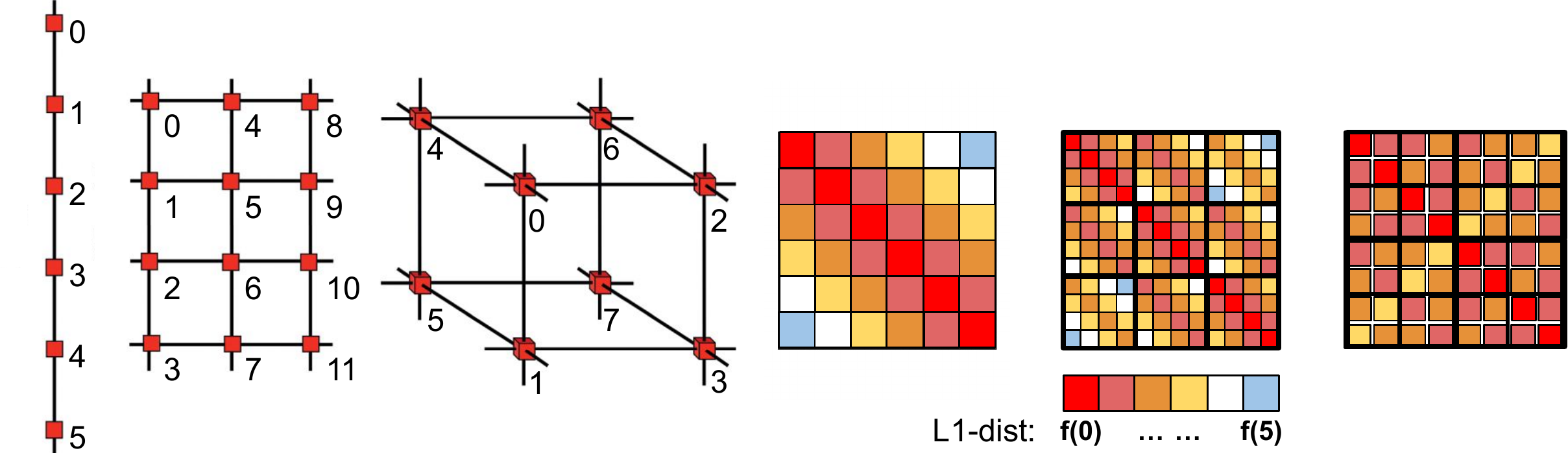}
    \caption{\small{\textbf{RPEs \& Beyond:} The $(i,j)$-entry of the regular RPE-mask is a (learnable) function $f$ of the distance $i-j$ between the ith and jth token in the input sequence that can be interpreted as a 1d-grid, thus it has the so-called \textit{Toeplitz} structure (first graph and colored-matrix in the figure). The proposed $d$-dimensional RPE acts on the $d$-dimensional grid input with the length of the shortest path $d(i,j)$ between node $i$ and $j$ in the gird replacing expression $i-j$ in the corresponding mask ($d=2$ includes image input and $d=3$, video input, see: second graph/matrix and third graph/matrix respectively). The corresponding mask is no longer Toeplitz, but is \textit{$d$-level block-Toeplitz}. Interestingly, all these matrix classes support fast matrix-vector multiplication (via Fast Fourier Transform) and thus, based on our first result, corresponding masked low-rank attention can be performed in sub-quadratic time (see Sec. \ref{sec:block-toeplitz}).}}
\label{fig:grid}
\vspace{-3mm}
\end{figure*}

Masking is a 
powerful mechanism altering the 
attention matrix by incorporating structural inductive bias. Flagship examples include (1) \textit{causal attention}, applied in generative Transformers \citep{xlnet}, where the arrow of time induces token-ordering with tokens not attending to their successors in the sequences, (2) \textit{relative positional encoding} (RPE, \citealp{rpe-base}) reducing interactions between distant tokens (but via a much more general mechanism than local attention) and (3) \textit{graph attention} incorporating topological signal from the graph \cite{yingcai, velickovic}. RPE-mechanisms were shown to significantly improve speech models \cite{speech1, zhou} and masks obtained from shortest-path length matrices were recently demonstrated to close the gap between the best customized graph neural networks models and Transformers \cite{short-path}. Straightforward application of the masking mechanism requires materialization of the attention matrix and consequently - impractical  quadratic time complexity for long input sequences (or large graphs).

In this paper we aim to answer the following question:
\textit{Under which conditions can masking be incorporated into attention mechanisms in a scalable way, i.e. in sub-quadratic time complexity in the number of input tokens?}

So far this question was answered only partially. Causality was incorporated in linear time into linear low-rank attention via the so-called \textit{prefix sum mechanism} by \citet{choromanski}. The same was proven recently for the special class of 
\textit{stochastic RPEs} \cite{liutkus}. Even more recently, a log-linear algorithm (applying Fast Fourier Transform) for incorporating general RPEs into low-rank attention was given by \citet{rpe-performers}. All these results leverage low-rank attention since so far that was the only known scalable mechanism which can approximate in particular regular dense softmax attention. 
Hence, the starting point of our analysis is also a low-rank attention model. 

Our contributions in this paper are as follows:
\vspace{-2.5mm}
\begin{enumerate}
    \item We answer the above question in Sec. \ref{sec:base_theory} by providing a surprisingly simple characterization of the efficient masking mechanisms: \textit{as long as the masking-matrix (element-wise multiplied with the regular attention matrix) supports sub-quadratic matrix-vector multiplication, the corresponding mechanism can be incorporated into low-rank attention in sub-quadratic time}. Interestingly, as we explain later, this result includes all mentioned partial results as special cases. 
    \vspace{-1.5mm}
    \item We present multiple consequences, leading in particular to novel scalable $d$-dimensional RPE mechanisms that can be applied in image and video processing (see Fig. \ref{fig:grid} and Sec. \ref{sec:block-toeplitz}), efficient implementations for the low-rank attention with \textit{padding} and \textit{packing} mechanisms (in common practical use for regular Transformers) (Sec. \ref{sec:base_theory}), and new scalable graph-based attention masking applying shortest-path signal and graph-diffusion kernels (Sec. \ref{sec:trees}, Sec. \ref{sec:graphs}). 
    \vspace{-1.5mm}
    \item Using our developed theory, we introduce a new masked-attention ML-model called \textit{graph kernel attention Transformer} (GKAT, Sec. \ref{sec:gkat}), and conduct comprehensive comparisons against \textbf{nine} other SOTA graph neural networks (GNNs) in Sec. \ref{sec:experiments}.
\end{enumerate}

We cast the masking problem as a topological (graph-based) modulation of unmasked attention, and leverage many mathematical techniques ranging from spectral analysis through dynamic programming on trees and random walks to new algorithms for solving Markov processes on graphs. 
The proofs of all theoretical results are given in the Appendix.

\vspace{-1mm}
\section{Preliminaries}
\label{sec:preliminaries}

We introduce notation used throughout the paper.

Denote by $L$ the number of input tokens. The attention used in a regular Transformer linearly projects their representations into three learnable matrices $\mathbf{Q}, \mathbf{K} \in \mathbb{R}^{L \times d_{QK}}$, $\mathbf{V} \in \mathbb{R}^{L \times d}$ called \textit{queries}, \textit{keys} and \textit{values} respectively. 

\begin{definition}[general masked attention]
\label{gen_graph_attention}
\textit{General masked softmax attention} is of the following form, where $\mathbf{N} \in \mathbb{R}^{L \times L}$ is the \textit{logits-mask}, and $\mathbf{A} \in \mathbb{R}^{L \times L}$ is the so-called \textit{masked attention matrix} (MAM):
\vspace{-1.5mm}
\begin{align}
\begin{split}
\label{eq:attnorm1}
    \mathrm{Att}_{\mathrm{SM}}(\mathbf{Q}, \mathbf{K}, \mathbf{V},\mathbf{N}) = \mathbf{D}^{-1} \mathbf{A} \mathbf{V},  \\
    \mathbf{A} = \exp (\mathbf{N} + \mathbf{Q} \mathbf{K}^\top / \sqrt{d_{QK}}), \quad \mathbf{D} = \mathrm{diag} ( \mathbf{A} \mathbf{1}_L ). 
\vspace{-3mm}
\end{split}  
\end{align}    
Here $\exp (\cdot)$ is applied element-wise, $\mathbf{1}_L$ is the all-ones vector of length $L$, and $\mathrm{diag} (\cdot)$ is a diagonal matrix with the input vector as the diagonal. The time complexity of computing (\ref{eq:attnorm1}) is $O(L^2 d)$.
The above is a special instantiation of the \textit{general masked kernel attention} which is defined as:
\vspace{-1.5mm}
\begin{align}
\label{eq:attnorm2}
\begin{split}
    \mathrm{Att}_{\mathrm{K}}(\mathbf{Q}, \mathbf{K}, \mathbf{V},\mathbf{M}) = \mathbf{D}^{-1} \mathbf{A} \mathbf{V},  \\
    \mathbf{A} = \mathbf{M} \odot \mathcal{K}(\mathbf{Q},\mathbf{K}), \quad \mathbf{D} = \mathrm{diag} ( \mathbf{A} \mathbf{1}_L ), 
\end{split}
\end{align}
where $\odot$ denotes the element-wise (Hadamard) matrix product,  $\mathrm{K}:\mathbb{R}^{d} \times \mathbb{R}^{d} \rightarrow \mathbb{R}$ is some kernel function and $\mathcal{K}(\mathbf{Q},\mathbf{K})$ is a kernel matrix defined as: $\mathcal{K}(\mathbf{Q},\mathbf{K})_{i,j} = \mathrm{K}(\mathbf{q}_{i}^{\top},\mathbf{k}_{j}^{\top})$ for the $ith$ row $\mathbf{q}_{i}$ of $\mathbf{Q}$ and the jth row $\mathbf{k}_{j}$ of $\mathbf{K}$ respectively.
We call $\mathbf{A}^{\prime} = \mathcal{K}(\mathbf{Q},\mathbf{K})$ the unmasked attention matrix (UAM).
The softmax attention can be obtained from the kernel one by taking: $\mathrm{K}(\mathbf{x},\mathbf{y}) \overset{\mathrm{def}}{=} \mathrm{exp}(\frac{\mathbf{x}^{\top}\mathbf{y}}{\sqrt{d_{QK}}})$ (the so-called \textit{softmax kernel}) and $\mathbf{M} \overset{\mathrm{def}}{=} \mathrm{exp}(\mathbf{N})$ (element-wise exponentiation).
\end{definition}

Low-rank attention methods provide (approximate) attention computation in time linear in the length $L$ of the input sequence if no masking is applied (i.e. $\mathbf{M}$ is all-ones) and kernel $\mathrm{K}$ admits (at least in expectation) a dot-product decomposition, i.e. 
$
\mathrm{K}(\mathbf{x}, \mathbf{y}) = \mathbb{E}[\phi(\mathbf{x})^{\top}\phi(\mathbf{y})]
$
for some (usually randomized) mapping: $\phi: \mathbb{R}^{d_{QK}} \rightarrow \mathbb{R}^{m}$ (and some $m >0$).
Such a decomposition (in fact more than one!) exists in particular for the softmax kernel used in most applications of regular Transformers.
We call $\phi(\mathbf{u})$ a \textit{(random) feature map} (RFM) for $\mathbf{u} \in \mathbb{R}^{d}$. 
For $\mathbf{Q}^{\prime},\mathbf{K}^{\prime} \in \mathbb{R}^{L \times m}$ with rows given as $\phi(\mathbf{q}_{i}^{\top})^{\top}$ and $\phi(\mathbf{k}_{i}^{\top})^{\top}$ respectively,
RFM-based kernel linearization leads directly to the efficient unmasked attention mechanism of the form:
\begin{align}
\begin{split}
    \widehat{\mathrm{Att}_\mathrm{K}} (\mathbf{Q}, \mathbf{K}, \mathbf{V}) = \widehat{\mathbf{D}}^{-1} (\mathbf{Q}^{\prime}((\mathbf{K}^{\prime})^{\top} \mathbf{V})), \\
    \quad \widehat{\mathbf{D}} = \mathrm{diag} (\mathbf{Q}^{\prime}((\mathbf{K}^{\prime})^{\top} \mathbf{1}_L) ). \label{performers_attention}
\end{split}    
\end{align} 
Here $\widehat{\mathrm{Att}_{\mathrm{K}}}$ stands for the approximate attention and brackets indicate the order of computations. It is easy to see that such a mechanism is characterized by time complexity $O(Lmd)$ as opposed to $O(L^{2}d)$ for regular attention. If $m \ll L$, computational gains are obtained.

\vspace{-3mm}
\section{Fast matrix-vector product is all you need}
\label{sec:base_theory}

Our first result, a natural extension of the theoretical analysis by \citet{rpe-performers}, shows that as long as mask $\mathbf{M} \in \mathbb{R}^{L \times L}$ supports sub-quadratic matrix-vector multiplication, it can be incorporated into low-rank attention in sub-quadratic time. This is explained in Lemma \ref{first_mask_lemma}.

\begin{algorithm}[H]
\caption{General Efficient Low-Rank Masked Attention}
\textbf{Input:}  Query/key matrices: $\mathbf{Q},\mathbf{K} \in \mathbb{R}^{L \times d_{QK}}$, value matrix $\mathbf{V} \in \mathbb{R}^{L \times d}$, mask $\mathbf{M} \in \mathbb{R}^{L \times L}$, procedure $\mathrm{FastMult}_{\mathbf{M}}:\mathbb{R}^{L} \rightarrow \mathbb{R}^{L}$ calculating $\mathbf{Mx}$ (or its approximation) for the input $\mathbf{x} \in \mathbb{R}^{L}$, kernel feature map: $\phi:\mathbb{R}^{d_{QK}} \rightarrow \mathbb{R}^{m}$.  $\mathrm{vec}(\cdot)$ denotes vectorization. \; \\
\textbf{Output:} Masked low-rank attention embeddings using $\phi$. \; \\
1. Compute matrices $\mathbf{V}^{1} \in \mathbb{R}^{L \times (md)}$, $\mathbf{V}^{2} \in \mathbb{R}^{L \times m}$ with rows defined as:
$\mathbf{V}^{1}_{i:}=\mathrm{vec}(\phi(\mathbf{k}_{i}^{\top})\mathbf{v}_{i})$, $\mathbf{V}^{2}_{i:}=\phi(\mathbf{k}_{i}^{\top})^{\top}$, where $\mathbf{k}_{i}$/$\mathbf{v}_{i}$ stands for the ith row of $\mathbf{K}$/$\mathbf{V}$. \; \\
2. Take $\tilde{\mathbf{D}}^{1}=[\mathrm{FastMult}_{\mathbf{M}}(\mathbf{V}^{1}_{:1}),...,\mathrm{FastMult}_{\mathbf{M}}(\mathbf{V}^{1}_{:md})] \in \mathbb{R}^{L \times md}$, $\tilde{\mathbf{D}}^{2} = [\mathrm{FastMult}_{\mathbf{M}}(\mathbf{V}^{2}_{:1}),...,\mathrm{FastMult}_{\mathbf{M}}(\mathbf{V}^{2}_{:m})] \in \mathbb{R}^{L \times m}$ for $\mathbf{V}^{1/2}_{:i}$ denoting ith column of $\mathbf{V}^{1/2}$.\; \\
3. Output the embedding $\mathbf{r}_{i}$ of the ith tokens as:
$\mathbf{r}_{i} = \frac{\phi(\mathbf{q}_{i}^{\top})^{\top}\mathrm{devec}(\tilde{\mathbf{D}}^{1}_{i:})}{\phi(\mathbf{q}_{i}^{\top})^{\top}(\tilde{\mathbf{D}}^{2}_{i:})^{\top}}$, where $\mathbf{q}_{i}$ is the ith row of $\mathbf{Q}$ and $\mathrm{devec}(\cdot)$ devectorizes its input back to $\mathbb{R}^{m \times d}$. 
\label{alg:main}
\end{algorithm}

\begin{lemma}[Tractable Mask Lemma]
\label{first_mask_lemma}
Assume that mask $\mathbf{M} \in \mathbb{R}^{L \times L}$ from Definition \ref{gen_graph_attention} supports matrix-vector multiplication in time $T_{\mathbf{M}}(L)$. Then the general masked kernel attention algorithm with mask $\mathbf{M}$ can be implemented in time $O((T_{\mathbf{M}}(L)+L)md)$.
\end{lemma}
\vspace{-2.5mm}
The algorithm is given in the algorithmic box 1.
We analyze it in the Appendix, but the intuition is that, in the unmasked low-rank setting, attention embeddings could be obtained from the action of $\phi(\mathbf{q}_{i})$ on the fixed (token-independent) matrix of shape $\mathbb{R}^{m \times d}$ summarizing all the tokens, whereas in the masked case the matrix depends on each token, but can be obtained from mask-vector products.  
\vspace{-3mm}
\paragraph{Causal attention:} Note that the prefix-sum algorithm from \cite{choromanski} is a special instantiation of Algorithm 1. Indeed, causality is encoded by the lower-triangular mask $\mathbf{M}$ such that: $\mathbf{M}_{i,j} = 1$ for $j \leq i$ and $\mathbf{M}_{i,j}=0$ otherwise. Every product $\mathbf{Mx}$ is trivially a vector of prefix sums: $\mathbf{x}_{1}+...+\mathbf{x}_{i}$ for $i=1,...,L$ and thus can be computed in time $O(L)$. 
\vspace{-3mm}
\paragraph{Packing \& Padding:} Both masking mechanisms are standard Transformers' techniques used to optimize attention computation on TPUs. The former packs multiple sequences in one \textit{super-sequence}. Mask is used here to prevent cross-sequence attention. The latter adds \textit{fake} tokens at the end of the legitimate input sequence (used if input's length varies). Mask is used here to prevent attention to fake tokens. Both masks $\mathbf{M}$ trivially support linear matrix-vector multiplication (see: Fig. \ref{fig:padding_packing}) and thus both packing and padding can be incorporated into low-rank attention in time linear in $L$.

\begin{figure}[h]
    \includegraphics[width=.99\linewidth]{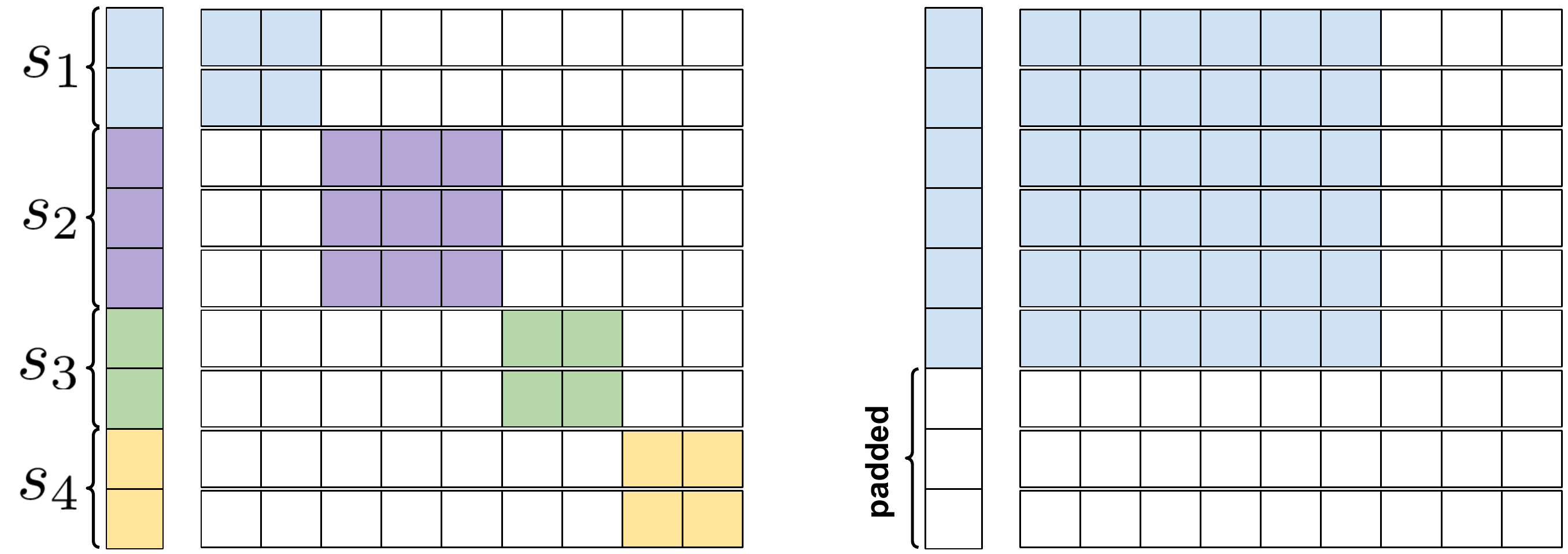}
    \caption{\small{\textbf{Left}: Padding with the super-sequence consisting of four sequences and its corresponding mask $\mathbf{M}$. \textbf{Right:} Packing with three fake padded tokens and its corresponding mask $\mathbf{M}$. For both masks, colored entries are equal to one and non-colored are equal to zero. Both masks trivially support linear matrix-vector multiplication.}}
\label{fig:padding_packing}
\end{figure}
\vspace{-4mm}
\subsection{Mask $\mathbf{M}$ as graph topology encoder}
From now on we will think about mask $\mathbf{M} \in \mathbb{R}^{L \times L}$ as a weighted adjacency matrix $\mathrm{Adj}(G)$ of some weighted graph $G=(V,E,W)$ with nodes/vertex-set $V$ of size $L$, edge set $E$ and edge-weight function $W:E \rightarrow \mathbb{R}$. Lemma \ref{first_mask_lemma} combined with this observation leads to several far-reaching conclusions regarding efficient incorporation of various masking mechanisms to Transformers, to which we devote the remaining part of our theoretical analysis. 

\subsection{$D$-dimensional Relative Positional Encodings}
\label{sec:block-toeplitz}

We need the following definition:
\begin{definition}[block-Toeplitz matrices]
\label{def:block-toeplitz}
We say that a matrix $\mathbf{M} \in \mathbb{R}^{L \times L}$ is Toeplitz (or 1-level block Toeplitz)
if there exists some $\xi:\mathbb{Z} \rightarrow \mathbb{R}$ such that $\mathbf{M}_{i,j}=\xi(i-j)$.
We say that $\mathbf{M} \in \mathbb{R}^{L \times L}$ is $d$-level block-Toeplitz for $d \geq 2$ if
$\mathbf{M}=(\mathbf{B}^{i,j})$ consists of block-matrices $\mathbf{B}^{i,j}$ taken from some set $\{\mathbf{A}_{1},...,\mathbf{A}_{r}\}$ of $(d-1)$-level block-Toeplitz matrices and if each block $\mathbf{B}^{i,j}$ is replaced with the index $k$ of its corresponding matrix $\mathbf{A}_{k}$, a Toeplitz matrix is obtained. 
\end{definition}

Consider the unweighted (i.e. all-one edge-weights) 1d-grid graph $G_{\mathrm{base}}$ (see: left graph in Fig. \ref{fig:grid}) and a complete graph (i.e. with all possible edges) $G$ obtained from it by defining each weight as $W_{{i,j}} = f(\mathrm{dist}_{G_{\mathrm{base}}}(i,j))$ for some (learnable) function $f:\mathbb{N} \rightarrow \mathbb{R}$ and where $\mathrm{dist}_{G_{\mathrm{base}}}(i,j)$ is the length of the shortest path between $i$ and $j$ in $G_{\mathrm{base}}$. If we define $\mathbf{M}=\mathrm{Adj}(G)$ then we get the regular RPE mechanism with the 1d-graph interpreted as the input sequence. 

Note that $\mathbf{M}$ defined in such a way is Toeplitz and thus supports $O(L\log(L))$ matrix-vector multiplication via Fast Fourier Transform (FFT). Thus low-rank RPE-masked attention can be conducted in $O(Ldm\log(L))$ time. This was the observation of \citet{rpe-performers}. What if we replace the 1d-grid with the d-dimensional grid and define $\mathbf{M}$ in the analogous way? The idea is to maintain the initial structure of the  topologically more complicated input, e.g. 2d-grid for images (with nodes as patches or even individual pixels) or 3d-grid for videos (with 2d-slices as different frames).

There is a particularly elegant answer to this question:
\begin{lemma}[$d$-dimensional RPEs]
\label{drpe-lemma}
Consider the generalized RPE-mechanism for the $d$-dimensional grid input defined above. Then there exists an ordering of input nodes such that $\mathbf{M}$ is a $d$-level block-Toeplitz matrix (see: Fig. \ref{fig:grid}).
\end{lemma}
Since $d$-level block-Toeplitz matrices support $O(L\log(L))$ matrix-vector multiplication via FFT for any fixed constant $d$ (see \citealp{block-toeplitz}), Lemma \ref{drpe-lemma} immediately leads to efficient corresponding masked attention computation.

\subsection{More general graph-masks using shortest-paths}
\label{sec:trees}
So far $G_{\mathrm{base}}$ was assumed to have a grid structure. What if we replace it with an arbitrary weighted graph? The following natural question arises:
\textit{Which condition does $G_{\mathrm{base}}$ and mapping $f:\mathbb{R} \rightarrow \mathbb{R}$ need to satisfy for the mask $\mathbf{M} \overset{\mathrm{def}}{=} [f(\mathrm{dist}_{G_{\mathrm{base}}}(i,j))]_{i,j=1,...,L}$ to support sub-quadratic matrix-vector multiplication ?}

We call such a pair $(G_{\mathrm{base}}, f)$ \textit{tractable}.
From what we have said so far, we conclude that:
\begin{corollary}
If $G_{\mathrm{base}}$ is an unweighted grid (of any dimensionality) then $(G_{\mathrm{base}}, f)$ is tractable $\forall f:\mathbb{R} \rightarrow \mathbb{R}$.
\end{corollary}
In several bioinformatics applications, e.g. molecular assembly trees \cite{assembly-trees},  the underlying input's topology is a forest (e.g. a tree). We prove the following:
\begin{lemma}
\label{tree-lemma}
If $G_{\mathrm{base}}$ is a forest and $f(z)=\exp(\tau(z))$ for affine mapping $\tau$, then $(G_{\mathrm{base}},f)$ is tractable and the related mask supports linear matrix-vector multiplication.
\end{lemma}
\vspace{-2.0mm}
\textit{Sketch of the proof:} The efficient algorithm for computing $\mathbf{w}=\mathbf{Mx}$ in this case is an application of the dynamic programming method for rooted trees. The algorithm first computes for each node $i$ the following expression: $s_{i} = \sum_{j \in \mathcal{T}_{i}} \exp(\tau(\mathrm{dist}(i,j)))\mathbf{x}_{j}$, where $\mathcal{T}_{i}$ stands for the subtree rooted in $i$ (in the bottom-up fashion from leaves to the fixed root). This is followed by the computation of the following expression: $\mathbf{w}_{i} = \sum_{j \in \mathcal{T}} \exp(\tau(\mathrm{dist}(i,j)))\mathbf{x}_{j}$ for every node in the order from the root to the leaves (leveraging already computed $s_{i}$). Details are given in the Appendix and computations are illustrated in Fig. \ref{fig:tree}.

We find a comprehensive description of tractable $(\mathrm{G}_{base}, f)$ an exciting analytic and combinatorial open problem.

\begin{figure}[h]
    \includegraphics[width=.99\linewidth]{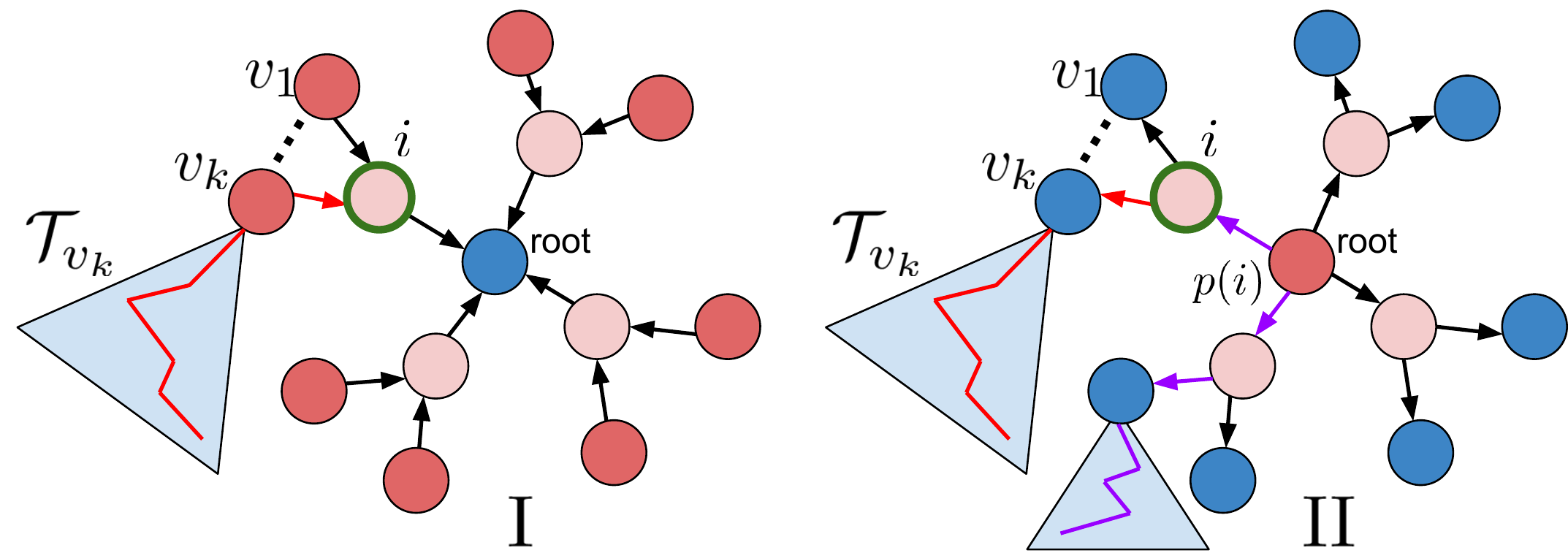}
    \caption{\small{Illustration of sketch of the proof of Lemma \ref{tree-lemma}. The directions of arrows show computation-flow. In phase I, $s_{i}$-terms are calculated in bottom-up fashion (from leaves to the root). The value of $s_{i}$ involving paths in $i$-rooted subtrees (red path with discarded directions) is updated based on $s_{v_{k}}$-terms involving paths in subtrees $\mathcal{T}_{v_{k}}$. To complete calculations, in phase II paths to nodes outside of the $i$-rooted tree are considered (purple path with directions discarded). Their contribution is calculated from the already computed $\mathbf{w}_{p(i)}$ for the parent $p(i)$ of node $i$ and $s_{i}$.}}
\label{fig:tree}
\end{figure}
\vspace{-3.5mm}
\subsection{Low-rank masking}
\label{sec:low-rank}
Note that in all previously considered cases, mask $\mathbf{M}$ is in general full-rank. However in several applications $\mathbf{M}$ can be assumed to have (at least in expectation) a low-rank decomposition, i.e.: $\mathbf{M}=\mathbb{E}[\mathbf{M}_{1}\mathbf{M}_{2}]$ for some (random) $\mathbf{M}_{1} \in \mathbb{R}^{L \times r}$, $\mathbf{M}_{2} \in \mathbb{R}^{r \times L}$ and some $0 < r \ll L$.
A flagship example is the \textit{stochastic RPE} mechanism presented in \cite{liutkus} corresponding to (logits-added) dot-product mask $\mathbf{N}$ translating to the softmax-kernel values mask $\mathbf{M}$. The latter one can be low-rank decomposed using any random feature map based softmax-kernel linearization mechanism, e.g. from \cite{choromanski}. In such a case, matrix-vector product $\mathbf{v}=\mathbf{Mx}$ can be computed (approximately) as: $\tilde{\mathbf{v}} = (\mathbf{M}_{1}(\mathbf{M}_{2}\mathbf{x}))$ in time $O(Lr)$, leading to overall time complexity $O(Lmrd)$ of the attention module using mask $\mathbf{M}$.

\section{Masking with graph kernels}
\label{sec:graphs}

Masks defined by shortest paths were shown to provide effective inductive bias for graph data (see \citealp{yingcai}), yet they cannot be interpreted as applying any valid kernel function on graph nodes and are very sensitive to small changes in the graph. A prominent class of kernel-functions $\mathrm{K}:V \times V \rightarrow \mathbb{R}$ defined on pairs of graphs nodes is the family of  \textit{graph-diffusion} or \textit{heat} kernels (GDKs). Thus it is natural to identify masks $\mathbf{M}$ for input graph data $G$ with the graph diffusion kernel matrices $\mathcal{K}_{\mathrm{K}}=[\mathrm{K}(i,j)]_{i,j=1,...,L}$. GDK is defined, for a hyperparameter $\lambda>0$ and $\mathbf{X}^{i}$ denoting the ith power of matrix $\mathbf{X}$, as:
\begin{equation}
\mathcal{K}_{\mathrm{K}} = \exp(-\lambda \mathbf{T}) \overset{\mathrm{def}}{=}\sum_{i=0}^{\infty} \frac{(-\lambda)^{i}\mathbf{T}^{i}}{i!},
\end{equation}
where either: $\mathbf{T}=\mathbf{L}$ for the Laplacian matrix $\mathbf{L} = \mathbf{D} - \mathrm{Adj}(G)$ and $\mathbf{D}=\mathrm{diag}([\mathrm{deg}(i)]_{i=1}^{L})$; or $\mathbf{T}=\mathbf{L}\mathbf{D}^{-1}$ (normalized Laplacian case) or $\mathbf{T} = - \mathrm{Adj}(G)$.

GDK is related to the diffusion process \cite{kondor} which describes in particular heat propagation. In a vacuum, the solution of the partial differential heat equation is the Gaussian-kernel, and in graphs it leads to GDK. Nodes better connected with each other (graph diffusion kernel quantifies it via the number of different-length walks with longer walks exponentially-deprioritized) give rise to larger kernel values. Finally, $t=\frac{1}{\lambda}$ can be interpreted as time when the solution is taken. GDK imprints topological signal of the propagation medium via left heat-signature. As $t \rightarrow \infty$ the kernel ``flattens" and the topological signal is lost.

Direct computation of the GDK matrix $\mathcal{K}_{\mathrm{K}}$ is of $O(L^{3})$ time complexity, thus prohibitively expensive even for sparse input graphs. However, a key observation is that Lemma \ref{first_mask_lemma} teaches us that for efficient masked low-rank attention we only need to compute efficiently the action $\mathrm{exp}(-\lambda \mathbf{T})\mathbf{x}$ of $\mathcal{K}_{\mathrm{K}}$ on a given $\mathbf{x} \in \mathbb{R}^{L}$. This leads to our next result.
\begin{theorem}[scalable Laplacian-GDK masking]
\label{lapl-theory}
Let a mask $\mathbf{M}$ be defined as $\mathbf{M} = \exp(-\lambda \mathbf{A})$,
for $\mathbf{A} = \mathbf{L}$ or $\mathbf{A}=\mathbf{L}\mathbf{D}^{-1}$, where $\mathbf{L}$ is the Laplacian of the input graph, and $\mathbf{D}=\mathrm{diag}([\mathrm{deg}(i)]_{i=1}^{L})$. Then low-rank masked attention can be computed in time $\tilde{O}((|E|+L)\log(2+\|\mathbf{A}\|_{\mathrm{F}})md)$, where $\tilde{O}$ hides $\mathrm{polylog}(L)$ factors, $|E|$ is the number of graph edges and $\|\cdot\|_{\mathrm{F}}$ is the Frobenius norm.
\end{theorem}
The theorem is a consequence of Lemma \ref{first_mask_lemma} and Theorem 1.2 from \cite{orecchia}. We see that if $|E|=o(L^{2})$, the masked attention mechanism is sub-quadratic in $L$.

\paragraph{Low-rank attention \& Markov processes with random initial distributions:} As noted by \citet{orecchia}, the heat kernel matrix for $\mathbf{T}=\mathbf{L}\mathbf{D}^{-1}$ can be interpreted as the
probability transition matrix of the discrete-time random walk where first the number of steps $i$ is sampled from a Poisson distribution with mean $\lambda$, and then $i$ steps of the natural random walk are performed on $G$. Looking at Algorithm 1, we conclude that here the low-rank structure of attention enables us to incorporate the GDK mask by solving that process in $md$ initial (randomized) distributions over nodes (randomization coming from mapping $\phi$) rather than in all $L$ one-hot initial distributions (that would correspond to the reconstruction of the entire transition matrix).

\paragraph{Remark:} The literature on efficiently computing the actions of matrix-exponentials (which is our main focus in this section) is very rich \cite{mohy}, partially because of straightforward applications in the theory of differential equations \cite{dongpingli}. In principle, each of these methods can be used by our algorithm. 

\subsection{Hypercubes with graph-diffusion kernels}

If the underlying graph is a hypercube, then GDK with $\mathbf{T}=\mathbf{L}$ has a closed-form formula. The following is true \cite{kondor}: 
$
\mathrm{K}_{\mathrm{GDK}}(i,j) \propto (\mathrm{tanh} \lambda)^{\mathrm{dist}(i,j)}
$
for the hyperbolic tangent $\mathrm{tanh}$.
Thus, as in Sec. \ref{sec:block-toeplitz}, the corresponding mask $\mathbf{M}$ is block-Toeplitz and hypercube-induced GDK-masking can be incorporated into low-rank attention in $O(Lmd\log(L))$ time.

\begin{figure*}[h]
    \centering
    \includegraphics[width=135mm,scale = 0.8]{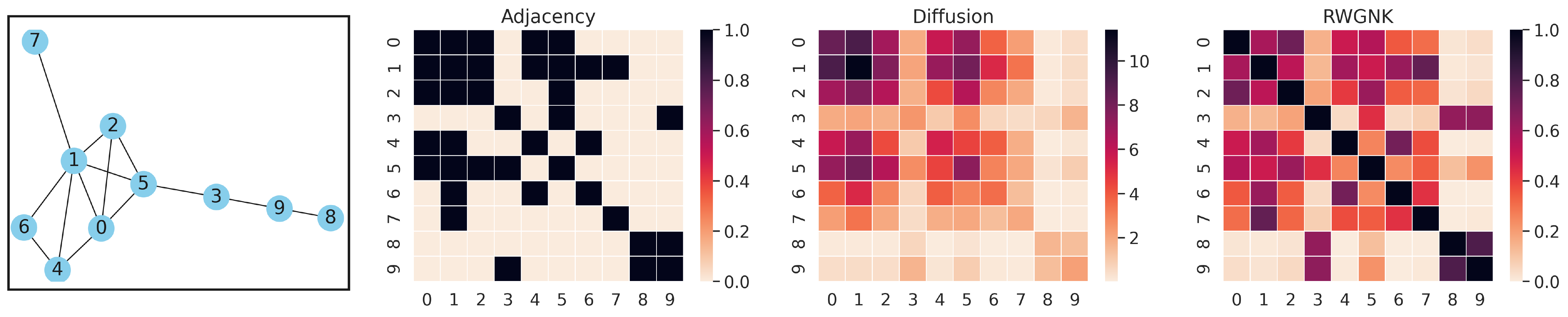}
    \vspace{-4.0mm}
    \caption{\small{From left to right: unweighted graph $\mathrm{G}$, its adjacency matrix $\mathrm{Adj}(\mathrm{G})$, its GDK matrix $\exp( \mathrm{Adj}(\mathrm{G}))$ and RWGNK-matrix with walk length of $3$ and $\alpha=1$. Colored cells measure the relationship among pairs of nodes (darker is stronger). The last two matrices can be thought of as continuous smoothings of $\mathrm{Adj}(\mathrm{G})$.}}
    \label{fig:heatmaps}
\end{figure*}

\subsection{Low-rank masking strikes back for GDKs}
\label{sec:gkat}

We now propose a proxy of the GDK with $\mathbf{T}=\mathrm{Adj}(G)$ such that the corresponding kernel matrix admits (in expectation) low-rank decomposition as in Sec. \ref{sec:low-rank}. Thus, based on the theory we developed, the corresponding mask $\mathbf{M}$ can be efficiently incorporated into low-rank attention with no need to call efficient solvers for the actions of matrix exponentials. We call our graph kernel the \textit{Random Walks Graph-Nodes Kernel} or RWGNK.

Intuitively, the value of the RWGNK for two nodes is given as a dot-product of two \textit{frequency vectors} that record visits in graph nodes of random walks beginning in the two nodes of interest. More formally, for the hyperparameters $\lambda, \alpha \geq 0,$ and two random walks $\omega(k)$, $\omega(l)$ with stopping probability $0 \leq p \leq 1$ (or of a fixed length) starting at $k$ and $l$ respectively, the RWGNK is given as:
\vspace{-3mm}
\begin{equation}
\label{rwgnk-eq}
\mathrm{K}^{\lambda,\alpha}_{p}(k,l)= \frac{\mathbb{E}_{\omega(k)}[f^{\omega(k),\lambda}_{k}]}{\|\mathbb{E}_{\omega(k)}[f^{\omega(k),\lambda}_{k}]\|_{2}^{\alpha}}\left(\frac{\mathbb{E}_{\omega(l)}[f^{\omega(l),\lambda}_{l}]}{\|\mathbb{E}_{\omega(l)}[f^{\omega(l), \lambda}_{l}]\|_{2}^{\alpha}}\right)^{\top}.     
\end{equation}
The (row) frequency vector $f^{\omega(h), \lambda}_{h}$ for $h \in \mathrm{V}$ is given as
$f^{\omega(h), \lambda}_{h}(i) \overset{\mathrm{def}}{=}\sum_{e \in \mathrm{L}^{\omega(h)}(i)} \lambda^{e}$, where $\mathrm{L}^{\omega(h)}(i)$ is the set of lengths of those \textit{prefix sub-walks} of a given random walk $\omega(h)$ that end at $i$ (where the prefix sub-walk of the walk $(j_{1},.j_{2},...,j_{t})$ is any walk of the form $(j_{1},...,j_{r})$ for some $r \leq t$ or an empty walk). Note that
Eq. \ref{rwgnk-eq} leads to the desired representation of $\mathrm{K}^{\lambda, \alpha}_{p}$ as
$\mathrm{K}^{\lambda, \alpha}_{p}(k,l)=\Psi(k)\Psi(l)^{\top}$, where $\Psi(h)$ is the renormalized expected frequency vector. In practice, expectations are replaced by Monte Carlo samplings over a few random walks, and vectors $\Psi(h)$ are not stored explicitly but in the form of weighted lookup tables. 

Figure \ref{fig:heatmaps} compares RWGNK-induced mask with the regular GDK-mask and the adjacency matrix mask. We call a Transformer applying low-rank masked attention via RWGNK, a \textit{Graph Kernel Attention Transformer} (or GKAT).

Next we explore the connection of RWGNKs with GDKs. We denote by $d_{\mathrm{max}},d_{\mathrm{min}}$ the maximum and minimum degree of a vertex in $\mathrm{G}$ respectively.
\begin{theorem}[RWGNKs count discounted numbers of walks]
\label{rwgnk_theorem}
The following is true for the kernel matrix $\mathcal{K}^{\lambda, \alpha}_ {p}(\mathrm{G})=[\mathrm{K}^{\lambda, \alpha}_{p}(k,l)]_{k,l \in \mathrm{V}(\mathrm{G})}$ of the $\mathrm{RWGNK}$ kernel with $0 \leq \lambda \leq 1$, $\alpha=0$ and $0 < p < 1$ for a graph $\mathrm{G}$ with vertex set $\mathrm{V}(\mathrm{G})$ of size $N$ (element-wise matrix inequality): 
\begin{equation}
\Gamma\left(\frac{\rho}{d_{\mathrm{max}}} \mathrm{Adj}(\mathrm{G})\right)
 \leq \mathcal{K}^{\lambda, \alpha}_{p}(\mathrm{G}) \leq \Gamma\left(\frac{\rho}{d_{\mathrm{min}}} \mathrm{Adj}(\mathrm{G})\right),
\end{equation}
\end{theorem}
where $\rho=(1-p)\lambda$ and $\Gamma(\mathbf{A})=\sum_{i=0}^{\infty}(i+1)\mathbf{A}^{i}$. 
Using the fact that $\mathrm{Adj}^{i}(\mathrm{G})$ encodes the number of walks of length $i$ between pairs of vertices in $\mathrm{G}$, we conclude that
$\mathrm{K}_{p}^{\lambda, 0}(k,l)
=\sum_{i=0}^{\infty} c^{i}_{k,l} r_{k,l}(i)$, where: $r_{k,l}(i)$ is the number of walks of length $i$ between nodes: $k$ and $l$
and $\frac{\sqrt[i]{i+1}(1-p)\lambda}{d_{\mathrm{max}}} \leq c_{k,l} \leq \frac{\sqrt[i]{i+1}(1-p)\lambda}{d_{\mathrm{min}}}$. Note that values of GDK with parameter $\lambda$ satisfy: 
$\mathrm{GDK}^{\lambda}(k,l) = \sum_{i=0}^{\infty} \tilde{c}^{i}(k,l) r_{k,l}(i)$, where: 
$\tilde{c}(k,l) = \frac{\lambda}{\sqrt[i]{i!}}$.
In practice, it suffices to have random walks of fixed length (instead of taking $p>0$) (see: Sec \ref{sec:experiments}). Furthermore, by taking $\alpha>0$ (e.g. $\alpha=1$) we can guarantee that kernel values are bounded.

\vspace{-2mm}
\section{Experiments}
\label{sec:experiments}
We focus on the GKAT architecture introduced in Sec. \ref{sec:gkat} as a prominent instantiation of the general mechanism presented in this paper and experiments with 2-level block-Toeplitz masking mechanisms introduced in  Sec. \ref{sec:block-toeplitz} for vision Transformers. 

Regarding GKAT, we conducted exhaustive evaluations on tasks ranging from purely combinatorial to bioinformatics, and benchmarked \textbf{10} different methods. 
All these experiments were run on a single Tesla P100 GPU with 16GB memory.
Experiments with vision Transformers were conducted on the ImageNet dataset.

\vspace{-3mm}
\subsection{Combinatorial Classification}
\label{sec:comb_c}
\vspace{-2mm}
In this section we focus on the problem of detecting local patterns in graphs. A model takes a graph $\mathrm{G}$ as an input and decides whether it contains some graph from the given family of graphs $\mathcal{H}$ as a subgraph (not necessarily induced) or is $\mathcal{H}$-free. This benchmark tests the abilities of different methods to solve purely combinatorial tasks.

\vspace{-1.0mm}
\small
\begin{figure}[h]
\vspace{-2.5mm}
    \begin{minipage}{1.0\textwidth}
    \includegraphics[width=.49\linewidth]{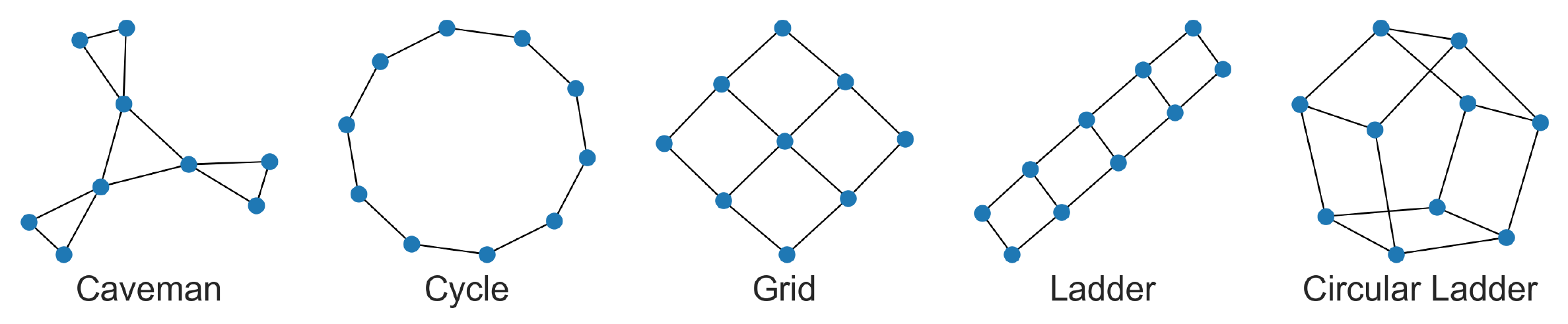}
    \end{minipage}
    \vspace{-4.5mm} 
    \caption{\small{Five motifs (patterns) used in the first class of combinatorial classification experiments. For each pattern $\mathrm{H}$, an algorithm is trained to distinguish between graphs $\mathrm{G}$ containing $\mathrm{H}$ and those that are $\mathrm{H}$-free. A naive brute-force algorithm for conducting this has time complexity $\Omega(N^{h})$, where $h$ is the number of nodes of the motif, prohibitively expensive for all these motifs (since $h \geq 9$).}}
\label{fig:five-motifs}
\end{figure}
\normalsize

\vspace{-4mm}
\subsubsection{Erd\H{o}s-R\'{e}nyi Random Graph with Motifs}
\label{sec:motif-detection}
\textbf{Data Generation:} Following the procedure from \cite{RWGNN2020}, we used five binary classification datasets
consisting of random Erd\H{o}s-R\'{e}nyi (ER) graphs connected with motifs (positive example) or other smaller ER graphs with the same average degree as a motif (negative example), see Fig. \ref{fig:five-motifs} (details in the Appendix, Sec. \ref{app:spantree_exp_details}). For each dataset we constructed $S=2048$ positive and $S$ negative examples.

\begin{figure*}[h]
    \centering
    \includegraphics[width=150mm,scale = 0.8]{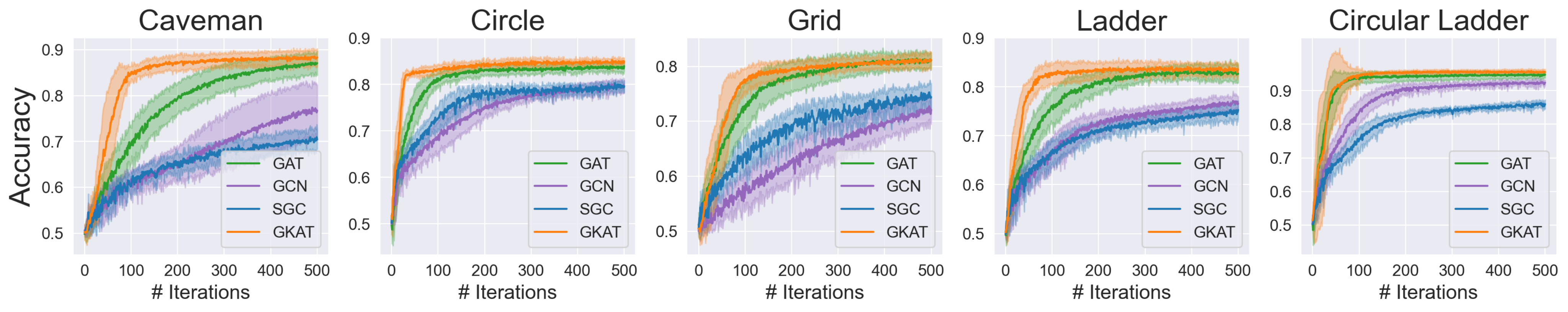}
    \vspace{-4.0mm}
    \caption{\small{Model accuracy comparison of all four methods: GKAT, GAT, GCN and SGC on the motif-detection task}. All architectures are $2$-layer. GKAT outperforms other algorithms on all the tasks. See also Appendix:Sec. \ref{large-graphs} for the tabular version with $100K$-size graphs.}
    \label{fig:motifs}
\end{figure*}

\textbf{Tested Algorithms \& Parameter Setting:} We tested our GKAT, graph convolution networks (GCNs, \citealp{kipf}), spectral graph convolution networks (SGCs, \citealp{bresson}) and graph attention networks (GATs, \citealp{velickovic}). A feature vector in each vertex was of length $l=5$ and contained top ordered $l$ degrees of its neighbors (if there were fewer than $l$ neighbors, we padded zeroes).
A dataset for each motif was randomly split into $75\%/25\%$ training/validation set. We chose: the number of epochs $E=500$, batch size $B=128$, used Adam optimizer with learning rate $\eta=0.001$ and early-stopped training if neither the validation loss nor validation accuracy improved for $c=80$ continuous epochs. 

We applied $2$-layer architectures.
For GCNs and SGCs, we used $h=32$ nodes in the hidden layer. For SGC, we furthermore bound each hidden layer with $2$ polynomial localized filters. For GAT and GKAT, we used $2$ attention heads, with $h=9$ nodes in the hidden layer to make all models of comparable sizes. In GKAT we used random walks of length $\tau=3$. The results are presented in Fig. \ref{fig:motifs}. GKAT outperforms all other methods for all 
the motifs.

\vspace{-2mm}
\subsubsection{Global graph properties \& Deep vs Dense}
\label{sec:induced_cycles}
\vspace{-2mm}
Next we took as $\mathcal{H}$ an infinite family of motifs rather than just a single motif. The algorithm needs to decide whether a graph contains an induced cycle of length $>T$ for a given constant $T$. Thus the motif itself became a global property that cannot be detected by exploring just a close neighborhood of a node. In this experiment we focused also on the ``depth versus density" trade-off. Shallow neural networks with dense attention are capable of modeling deeper networks relying on sparse layers, yet the price is extra computational cost per layer. We test here whether architectures that apply RWGNK kernels leveraging efficient decomposable long-range attention from Sec. \ref{sec:gkat} can also replace deeper counterparts or if they lose their expressiveness.

\vspace{-1.0mm}
\small
\begin{figure}[h]
\vspace{-3mm}
    \begin{minipage}{1.0\textwidth}
    \includegraphics[width=.49\linewidth]{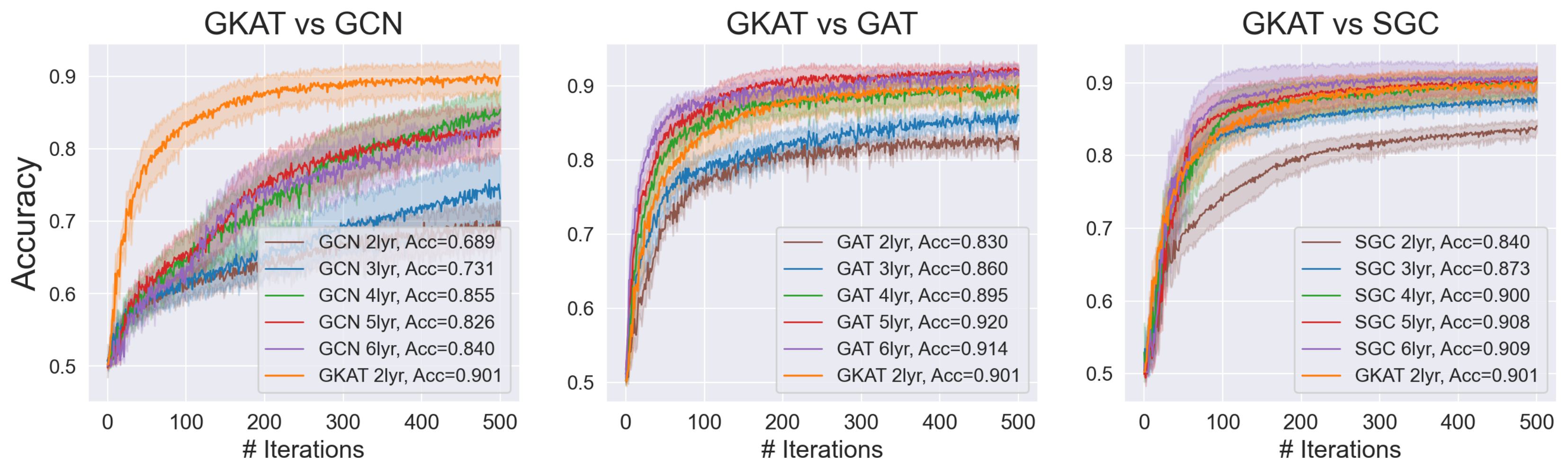}
    \end{minipage}
    \vspace{-4.5mm}
    \caption{\small{Comparison of the two-layer GKAT with different variants of GCNs, GATs and SGCs, varying by the number of hidden layers. Shallow GKAT architecture has the expressiveness of deeper version of its counterparts and in fact outperforms many of them (e.g. graph convolution networks.)}}
\label{fig:SpanTreeVal}
\end{figure}
\normalsize

\vspace{-2.5mm}
\textbf{Dataset Generation:} We created $S=2048$ random binary trees, each having $50$ nodes, with $75\%/25\%$ for training/validation. For each tree, we constructed a positive example, by connecting two nodes with the farthest distance from each other (a negative example was obtained by connecting two random vertices, but not farthest from each other). Note that a positive example constructed in such a way has shortest induced cycle of length $P+1$, where $P$ is the diameter of the tree.

\textbf{Tested Algorithms \& Parameter Setting:} We used the same algorithms as before and run detailed ablation studies on the depth of the GKAT competitors, by comparing two-layer GKAT with GATs, GCNs and SGCs of up to six layers. 

For a fair comparison, we used models with a comparable number of parameters. For the two-layer GKAT, we applied $8$ heads in the first layer, and $1$ head in the second layer. The dimension of each head was $d=4$. The last layer was fully-connected with output dimensionality $o=2$ for binary classification. We applied random walk length of $\tau=6$. For GCN, GAT and SGC, we tested number of layers ranging from $2$ to $6$. We controlled the number of nodes in the hidden layer(s) for GCN, GAT and SGC, and the number of attention heads in each head for GAT so that their total number of trainable parameters was comparable with that of our two-layer GKAT. All other parameters were chosen as in Sec. \ref{sec:motif-detection}.
More details on parameter settings and additional ablation tests over random walk length of GKAT are given in Table \ref{table:spantree_parameters_setting} and Fig. \ref{fig:spantree_GKAT_rwlen_ablation} in the Appendix (Sec. \ref{app:spantree_exp_details}). Our main results are presented in Fig. \ref{fig:SpanTreeVal}.

We see that a shallow two-layer GKAT beats all GCN-variants (also deeper ones) as well as GATs and SGCs with $<4$ layers by a wide margin. A two-layer GKAT is asymptotically equivalent to the four-layer GAT and SGC, yet as we show in Sec. \ref{sec:speed}, is faster to train and run inference on. 

\vspace{-3mm}
\subsection{Bioinformatics \& Social Networks experiments}
\label{sec:faircomparison}
\vspace{-1mm}
\textbf{Datasets:} We tested GKAT for graph classification tasks on 9 standard and publicly available bioinformatics and social networks datasets \citep{kersting2016benchmark} using a carefully designed model selection and assessment framework for a fair comparison \cite{FairComp}. The former include: D$\&$D \cite{DD_2003}, PROTEINS \cite{PROTEINS_2005}, NCI1 \cite{NCI1_2008} and ENZYMES \cite{ENZYMES_2004}, and the latter: IMDB-BINARY, IMDB-MULTI, REDDIT-BINARY, REDDIT-5K and COLLAB \cite{social_dataset_2015}, see also Sec. \ref{app:chem_social_dataset_descriptions}. 

\textbf{Tested Algorithms:} We compared GKAT with top GNN methods used previously for that data: DCGNN \cite{DGCNN_2018}, DiffPool \cite{DiffPool_2018}, ECC \cite{ECC_2017}, GraphSAGE \cite{GraphSAGE_2017} and RWNN \cite{RWGNN2020}, which are selected based on their popularity and architectural differences. For bioinformatics datasets, but ENZYMES, we used the Molecular Fingerprinting (MF, \citealp{Fingerprint_2005, Fingerprint_2019}) as a baseline. This first applies global sum pooling and then a single-layer MLP with ReLU activations. For social datasets and ENZYMES, we applied the DeepMultisets (DM) method \cite{Deepsets_2017} as a baseline. We did not add the numbers for the GIN method \cite{GIN_2019} since we could not reproduce its reported results for the models of size similar to GKAT.   

\vspace{-4.5mm}
\begin{table}[h]
    \begin{center}
    \caption{\small{Performance of different algorithms on the bioinformatics data. For each dataset, we highlighted/underlined the best/second best method. GKAT is the best on three out of four tasks.}\vspace{-3mm}}
    \scalebox{0.78}{
     \begin{tabular}{ c c  c  c  c  c } 
    \toprule
     & $\textbf{\textrm{D\&D}}$ & $\textbf{\textrm{NCI1}}$ & $\textbf{\textrm{Proteins}}$ & $\textbf{\textrm{Enzymes}}$  \\ [0.5ex] 
    \toprule
     $\textbf{\textrm{Baseline}}$ & $\underline{\textrm{78.4 \textpm 4.5\%}}$  & $\textrm{69.8\textpm2.2\%}$ & $\textbf{\textrm{75.8\textpm3.7\%}}$ & $\underline{\textrm{65.2\textpm6.4\%}}$ \\
    \toprule
     $\textbf{\textrm{DGCNN}}$ & $\textrm{76.6\textpm4.3\%}$ & $\underline{\textrm{76.4\textpm1.7\%}}$ & $\textrm{72.9\textpm3.5\%}$ & $\textrm{38.9\textpm5.7\%}$ \\
     $\textbf{\textrm{DiffPool}}$ & $\textrm{75.0\textpm3.5\%}$ & $\textbf{76.9\textpm1.9\%}$ & $\textrm{73.7\textpm3.5\%}$ & $\textrm{59.5\textpm5.6\%}$ \\
     $\textbf{\textrm{ECC}}$ & $\textrm{72.6\textpm4.1\%}$ & $\textrm{76.2\textpm1.4\%}$ & $\textrm{72.3\textpm3.4\%}$ & $\textrm{29.5\textpm8.2\%}$ \\
    
     $\textbf{\textrm{GraphSAGE}}$ & $\textrm{72.9\textpm2.0\%}$ & $\textrm{76.0\textpm1.8\%}$ & $\textrm{73.0\textpm4.5\%}$ & $\textrm{58.2\textpm6.0\%}$ \\
     $\textbf{\textrm{RWNN}}$ & $\textrm{77.6\textpm4.7\%}$ & $\textrm{71.4\textpm1.8\%}$ & $\textrm{74.3\textpm3.3\%}$ & $\textrm{56.7\textpm5.2\%}$ \\
    \bottomrule
     $\textbf{\textrm{GKAT}}$ & $\textbf{\textrm{78.6\textpm3.4\%}}$ & $\textrm{75.2\textpm2.4\%}$ & $\textbf{\textrm{75.8 \textpm 3.8\%}}$ & $\textbf{\textrm{69.7 \textpm 6.0\%}}$ \\
    \bottomrule
    \end{tabular}}
    \label{table:real_dataset_results_chemical}
    \end{center}
\vspace{-4mm}
\end{table}
\normalsize

\vspace{-4mm}
\begin{table}[h]
    \caption{\small{Performance of different algorithms on the social network data. GKAT is among two top methods for four out of five tasks.}\vspace{-5mm}}
    \begin{center}
    \scalebox{0.65}{
     \begin{tabular}[H]{ c c c c c c } 
    \toprule
     & $\textbf{\textrm{IMDB-B}}$ & $\textbf{\textrm{IMDB-M}}$ & $\textbf{\textrm{REDDIT-B}}$ & $\textbf{\textrm{REDDIT-5K}}$  & $\textbf{\textrm{COLLAB}}$\\ [0.5ex] 
    \toprule
     $\textbf{\textrm{Baseline}}$ & $\underline{\textrm{70.8\textpm5.0\%}}$  & $\textbf{\textrm{49.1 \textpm 3.5\%}}$ & $\textrm{82.2\textpm3.0\%}$ & $\textrm{52.2\textpm1.5\%}$ & $\textrm{70.2\textpm1.5\%}$ \\
    \toprule
     $\textbf{\textrm{DGCNN}}$ & $\textrm{69.2\textpm5.0\%}$ & $\textrm{45.6\textpm3.4\%}$ & $\textrm{87.8\textpm2.5\%}$ & $\textrm{49.2\textpm1.2\%}$ & $\textrm{71.2\textpm1.9\%}$ \\
     $\textbf{\textrm{DiffPool}}$ & $\textrm{68.4\textpm3.3\%}$ & $\textrm{45.6\textpm3.4\%}$ & $\textrm{89.1\textpm1.6\%}$ & $\underline{\textrm{53.8\textpm1.4\%}}$ & $\textrm{68.9\textpm2.0\%}$ \\
     $\textbf{\textrm{ECC}}$ & $\textrm{67.7\textpm2.8\%}$ & $\textrm{43.5\textpm3.1\%}$ & $\textrm{OOM}$ & $\textrm{OOM}$ & $\textrm{OOM}$ \\
     
     $\textbf{\textrm{GraphSAGE}}$ & $\textrm{68.8\textpm4.5\%}$ & $\textrm{47.6\textpm3.5\%}$ & $\textrm{84.3\textpm1.9\%}$ & $\textrm{50.0\textpm1.3\%}$ & $\textbf{\textrm{73.9\textpm1.7\%}}$ \\
     $\textbf{\textrm{RWNN}}$ & $\underline{\textrm{70.8\textpm4.8\%}}$ & $\underline{\textrm{47.8\textpm3.8\%}}$ & $\textbf{\textrm{90.4\textpm1.9\%}}$ & $\textrm{51.7\textpm1.5\%}$ & $\textrm{71.7\textpm2.1\%}$ \\
    \bottomrule
     $\textbf{\textrm{GKAT}}$ & $\textbf{\textrm{71.4\textpm2.6\%}}$ & $\textrm{47.5\textpm4.5\%}$ & $\underline{\textrm{89.3\textpm2.3\%}}$ & $\textbf{\textrm{55.3\textpm1.6\%}}$ & $\underline{\textrm{73.1\textpm2.0\%}}$ \\
    \bottomrule
     
    \end{tabular}}
     \label{table:real_dataset_results_social}
    \end{center}
\vspace{-3mm}
\end{table}
\normalsize
\vspace{-2mm}
\textbf{GKAT Setting:} We used a two-layer GKAT followed by the baseline layers: we first applied an attention layer with $k$ heads (a hyperparameter to be tuned), and then another one with one head to aggregate topological information on graphs. Next, we applied either the MF method or the DM method to further process the aggregated information. The random walk length $\tau$ in each GKAT layer satisfied $\tau \leq 4$ and depended on the evaluated datasets. The average graph diameter shown in Table \ref{table:dataset_statistics} in the Appendix helps to calibrate walk length. We chose it to balance the pros of using a shallow architecture and the cons of information loss from dense layer compression. GKAT increased the number of the baseline's parameters by a negligible fraction.

\textbf{Training Details:} We used a 10-fold CV for model assessment, and an inner holdout with $90\%/10\%$ training/validation split for model selection following the same settings \cite{FairComp}. We then trained the whole training-fold three times, randomly holding out $10\%$ of data for early stopping after model selection in each fold. The average score for these runs was reported in Table \ref{table:real_dataset_results_chemical} \&  \ref{table:real_dataset_results_social}. 

The results from Table \ref{table:real_dataset_results_chemical} and Table \ref{table:real_dataset_results_social} show that GKAT is the best on three out of four bioinformatics datasets and is among two best methods on four out of five social network datasets. It is the only GNN method that consistently outperforms baseline on all but one bioinformatics dataset (bio-data benefits more than others from efficient longer-range attention modeling as showed by \citet{choromanski}).
In the Appendix (Sec. \ref{app:citation_exp_details}) we provide additional comparisons of GKAT with GAT on citation networks, where GKAT outperforms GAT on two out of three tested datasets.
\vspace{-7mm}
\subsection{Space \& Time Complexity Gains of GKAT}
\label{sec:speed}
We measured speed and memory improvements coming from GKAT as compared to GAT as well as accuracy loss in comparison to Transformer using GKAT masking, but explicitly computed (GKAT-0), see Table \ref{table:speed}. We decided to report relative rather than absolute numbers since the former are transferable across different computational setups. The accuracy gaps of the corresponding GKAT-0 and GKAT models (obtained after the same \# of epochs) are marginal, yet GKAT yields consistent speed and memory gains as compared to GAT per attention layer, particularly substantial for very large graphs as those from \textrm{Citeseer} and \textrm{Pubmed}.

\begin{table}[h]
\vspace{-4.0mm}
    \caption{\small{Speed \& Space Complexity gains provided by GKAT. \textbf{First row:} memory compression (lower better).} \textbf{Second \& third row:} speedup in training and inference respectively per one attention layer as compared to GAT. \textbf{Last row:} accuracy loss as compared to GKAT-0 applying brute-force RWGNK masking. We used four datasets from Sec. \ref{sec:motif-detection}, a dataset from Sec. \ref{sec:induced_cycles} ($\mathrm{Tree}$) and two citation network datasets (see: Sec. \ref{app:citation_exp_details}): $\mathrm{Citeseer}$ and $\mathrm{Pubmed}$ with graphs of much larger sizes and on which GKAT also outperforms GAT. We applied $r=256$ random features to linearize softmax kernel for features in nodes for citation network datasets, $r=16/8$ for datasets from Sec. \ref{sec:motif-detection}/ \ref{sec:induced_cycles}.\vspace{1mm}} 
    \label{table:speed}
    \scalebox{0.57}{
     \begin{tabular}{ c c c c c c c c c c c} 
    \toprule
       & $\textbf{\textrm{}}$ & $\textbf{\textrm{Cavem.}}$ & $\textbf{\textrm{Circle}}$ & $\textbf{\textrm{Grid}}$ & $\textbf{\textrm{Ladder}}$ & $\textbf{\textrm{Tree}}$ 
       & $\textbf{\textrm{Citeseer}}$ & $\textbf{\textrm{Pubmed}}$ 
       \\ [0.5ex] 
    \toprule
       & \hspace{-7mm}$\textbf{\textrm{GKAT }/\textbf{ GKAT-0 memory}}$ & $\textrm{0.54}$ & $\textrm{0.53}$ & $\textrm{0.55}$ & $\textrm{0.52}$ & $\textrm{0.95}$ 
       & $\textrm{0.18}$ & $\textrm{0.07}$ 
       \\ [0.5ex] 
    \toprule
       & \hspace{-8mm} $\textbf{\textrm{train speedup vs GAT}}$ & $\textrm{1.40x}$ & $\textrm{1.41x}$ & $\textrm{1.42x}$ & $\textrm{1.40x}$ & $\textrm{1.10x}$ 
       & $\textrm{5.10x}$ & $\textrm{9.50x}$ 
       \\ [0.5ex] 
    \toprule
       & \hspace{-8mm} $\textbf{\textrm{inf speedup vs GAT}}$ & $\textrm{1.46x}$ & $\textrm{1.49x}$ & $\textrm{1.49x}$ & $\textrm{1.47x}$ & $\textrm{1.12x}$ 
       & $\textrm{5.21x}$ & $\textrm{9.54x}$ 
       \\ [0.5ex] 
    \toprule
       & \hspace{-6mm} $\textbf{\textrm{GKAT-0 - GKAT (accur.)}}$ & $\textrm{0.07\%}$ & $\textrm{0.09\%}$ & $\textrm{0.08\%}$ & $\textrm{0.07\%}$ & $\textrm{0.06\%}$ 
       & $\textrm{0.05\%}$ & $\textrm{0.06\%}$ 
       \\ [0.5ex] 
    \bottomrule
    \end{tabular}}

\end{table}
\normalsize
\vspace{-2.0mm}

In Table \ref{table:clocktime} we show that GKAT is also faster that its counterparts (GCN, GAT, SGC) in terms of wall clock time needed to reach particular accuracy levels, by comparing accuracy levels reached by different models in a given wall clock time budget (time GKAT needs to complete first $100$ epochs).

\vspace{-3.0mm}
\small
\begin{table}[h]
    \begin{center}
    \caption{\small{Running time of training different networks on datasets from Sec. \ref{sec:motif-detection} and Sec. \ref{sec:induced_cycles}. For GCN, GAT and SGC, we reported the accuracy with 2 layers. For GKAT, we used a 2-layer architecture and reported the accuracy with a fixed walk length of $6$ for Induced Cycle Detection, and of $3$ for motifs from Sec. \ref{sec:motif-detection}.}\vspace{-2mm}}
    \vspace{-1mm}
    \scalebox{0.65}{
     \begin{tabular}{ c  c  c  c  c  c  c} 
    \toprule
      & $\textbf{\textrm{Induced Cycle}}$ & {$\textbf{\textrm{Caveman}}$} & {$\textbf{\textrm{Circle}}$} & {$\textbf{\textrm{Grid}}$}  & {$\textbf{\textrm{Ladder}}$}  & {$\textbf{\textrm{Circle Ladder}}$} \\ [0.5ex] 
    \toprule
     $\textbf{\textrm{GCN}}$ & $63.2\%$  & $62.1\%$ & $71.4\%$ & $59.3\%$ & $66.7\%$ & $87.4\%$\\
     $\textbf{\textrm{GAT}}$ & $77.0\%$ & $69,1\%$ & $80.6\%$ & $73.8\%$ & $75.9\%$ & $93.7\%$\\
     $\textbf{\textrm{SGC}}$ & $56.6\%$ & $55.4\%$ & $64.7\%$ & $58.2\%$ & $59.1\%$ & $66.5\%$\\
    \bottomrule
     $\textbf{\textrm{GKAT}}$ & $\textbf{83.6\%}$ & $\textbf{85.1\%}$ & $\textbf{83.3\%}$ & $\textbf{77.1\%}$ & $\textbf{82.4\%}$ & $\textbf{94.6\%}$ \\
    \bottomrule
    \end{tabular}}
\label{table:clocktime}    
\end{center}
\end{table}
\normalsize

\subsection{2-level block-Toeplitz masks for vision data}

In this section (see: Fig. \ref{fig:blockt}), we present additional results in the vision domain, showing large, \textbf{+}$\mathbf{2.5}$-$\mathbf{3.4}$ percentage point, gains in accuracy arising from applying the 2-level block Toeplitz masking introduced in the paper (see: Sec. \ref{sec:block-toeplitz}) on top of the regular vision Performer. As explained before, the price we pay for these gains is only a $\log(L)$ multiplicative factor in time complexity.

\begin{figure}[h]
\vspace{-3.5mm}
    \begin{minipage}{1.0\textwidth}
    \includegraphics[width=.48\linewidth]{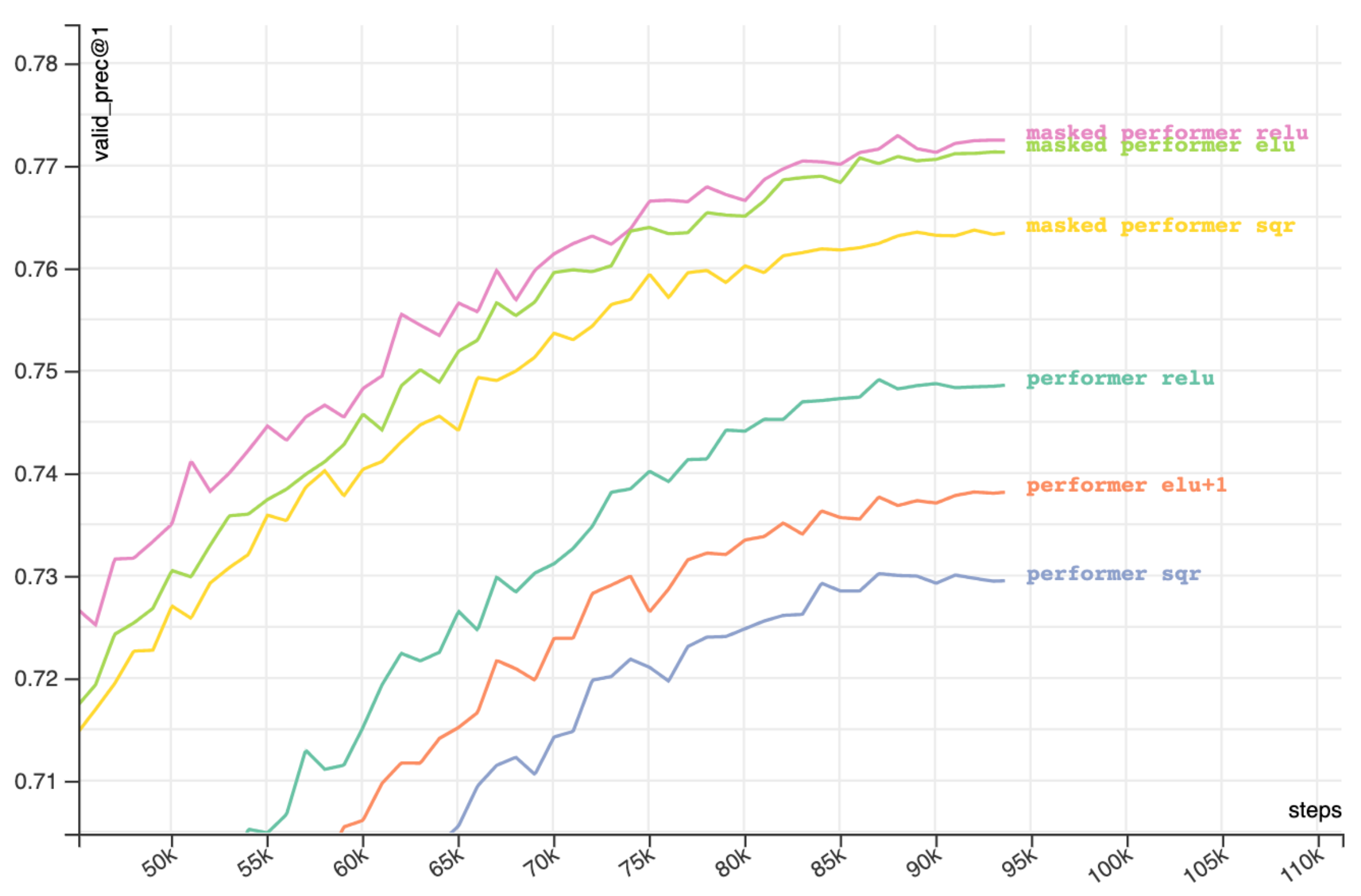}
    \end{minipage}
    \vspace{-4.5mm} 
    \caption{Comparison of regular Performers using $x^2$, $\mbox{ELU}+1$, and ReLU kernels, with their counterparts applying 2-level block Toeplitz masking from our paper on the ImageNet classification task (hidden size = 768, 12 layers \& heads, MLP dim = 3072).}
\label{fig:blockt}
\end{figure}

\vspace{-7mm}
\section{Additional results \& some open questions}
\label{sec:additional_results}
\vspace{-1mm}

In this section, we present additional theoretical results regarding the theory of the scalable efficient masked Transformers that we have developed. We focus on masking mechanisms for the tree-graph inputs. We also discuss some open problems for future work. 

In Section \ref{sec:trees} we showed that for the specific classes of functions $f$ (exponentials of affine mappings of the shortest-distance paths) and arbitrary weighted trees, pairs $(G_{\mathrm{base}}, f)$ are tractable. Here we will provide additional related results, but for arbitrary functions $f$. Our first result is as follows:

\begin{lemma}
\label{lemma:fft-tree}
If $\mathcal{T}=G_{\mathrm{base}}$ is an unweighted tree and $f$ is an arbitrary function then the corresponding mask matrix $\mathbf{M}=\mathbf{M}(G_{\mathrm{base}},f)$ supports matrix-vector multiplication in time $O(L \cdot \log^{2}(L))$. Thus $(G_{\mathrm{base}},f)$ is tractable.
\end{lemma}

\vspace{-1.5mm}
We now show that if the diameter $\mathrm{diam}(\mathcal{T})$ of the tree $\mathcal{T}$ is of the order of magnitude $o(\log^{2}(L))$, where $L=|V(\mathcal{T})|$, a more efficient algorithm can be used.

\begin{lemma}
\label{lemma:diam}
If $\mathcal{T}=G_{\mathrm{base}}$ is an unweighted tree and $f$ is an arbitrary function then the corresponding mask matrix $\mathbf{M}=\mathbf{M}(G_{\mathrm{base}},f)$ supports matrix-vector multiplication in time $O(L \cdot  \mathrm{diam}(G_{\mathrm{base}}))$. 
\end{lemma}

\begin{corollary}
From the above, we obtain an $O(L \log(L))$ algorithm for computing the action of $\mathbf{M}$ on $\mathbf{x}$ if $G_{\mathrm{base}}$ is a tree with: (a) a node of degree $\geq 2$ that has same-length paths to all the leaves and (b) all other nodes of degree $\geq 3$ (e.g. complete binary tree).
\end{corollary}

\begin{corollary}
The algorithm from the proof of Lemma \ref{lemma:diam} (see:Appendix) can be used to improve algorithm from Lemma \ref{lemma:fft-tree} (that works for graphs with arbitrary diameters) if the input $\mathcal{T}$ to that algorithm satisfies: $\mathrm{diam}(\mathcal{T})=o(\log^{2}(|V(\mathcal{T})|))$. Note that computing the diameter of any tree $\mathcal{T}$ can be done in time $O(|V(\mathcal{T})|)$ by running two depth-first-search  procedures: the first one from an arbitrary nodes $v$ of $\mathcal{T}$ and the second one from the node farthest from $v$ in $\mathcal{T}$ (obtained via the first depth-first-search procedure).
\vspace{-3.0mm}
\end{corollary}

We leave the Reader with an interesting open problem:

\textit{Can we improve Lemma \ref{lemma:fft-tree} to obtain $O(L\log(L))$ running time, i.e. replace the $\log^{2}(L)$ factor with a $\log(L)$ factor?}

Note that if the unweighted tree is a single path, the answer to the above question is: Yes. Indeed, this is precisely the 1D-RPE setting that we have discussed before. Furthermore, since in that setting the problem reduced to the multiplication with Toeplitz matrices, inherently relying on the FFT, the $\log(L)$ factor in all likelihood cannot be improved (unless FFT can be replaced with a faster algorithm or multiplication with  Toeplitz matrices is conducted approximately). Still, we do not know whether for general unweighted trees (or even nontrivial tree-extensions of the path) we can reach the $O(L\log(L))$ running time.

It might be also interesting to analyze how those of our presented methods that work for tree input data can be extended to non-tree graphs, but with low treewidth \cite{pilipczuk}, that can be thought of as relaxations of trees.

\vspace{-3.0mm}
\section{Conclusion}
\label{sec:conclusion}
\vspace{-1.0mm}
We presented a holistic approach to incorporating masking into scalable low-rank Transformers. We provided general theoretical results which include 
earlier results as special cases. 
We conducted comprehensive empirical evaluations of the new instantiations of the mechanism for graph data. 

We focused in the paper not only on scalable variants, but have introduced several new masking meethods that can be used on their own, even in regular Transformers. These include in particular d-level block-Toeplitz masking mechanisms with applications in vision and video processing, that we believe might lead to new vision-Transformers architectures. We show that topological masking is a powerful inductive bias and that corresponding ``topological Transformers'' turn out to be effective in various domains such as bioinformatics and vision.

\vspace{-3.0mm}
\section{Acknowledgements}
\vspace{-1.5mm}
AW acknowledges support from a Turing AI Fellowship under grant EP/V025279/1, The Alan Turing Institute, and the Leverhulme Trust via CFI.
\label{sec:acknowledgements}

\bibliographystyle{icml2022}
\newpage
\appendix
\onecolumn

\section{Appendix}

\subsection{Several Pointers}
\label{sec:code-pointers}

We include pointers to this part of the code that does not include sensitive/proprietary information. The core GKAT framework (with the additional analysis of new graph sketches that GKAT leads to, called \textit{graphots}, mixing regular feature vectors in nodes with the topological features) is here: \href{https://github.com/HL-hanlin/GKAT}{\textcolor{blue}{https://github.com/HL-hanlin/GKAT}}.  We used (deterministic and random) feature map mechanisms corresponding to the features defined in graph nodes from this repository: \href{https://github.com/google-research/google-research/tree/master/performer}{\textcolor{blue}{https://github.com/google-research/google-research/tree/master/performer}}.

\subsection{Combinatorial Classification Experiments: Additional Details}
\label{app:spantree_exp_details}

The data for the motif detection task from Section \ref{sec:motif-detection} was generated as follows:

\begin{itemize}
    \item Firstly we created five simple motifs as shown in Fig. \ref{fig:five-motifs}.
    Note that each motif has $\geq 9$ vertices so a brute-force combinatorial algorithm for motif-detection would take time $\Omega(N^{9})$, prohibitively expensive even for small graphs $\mathrm{G}$.
    \item We then generated for each motif $S$ small Erd\H{o}s-R\'{e}nyi graphs with the same number of nodes as that motif and the same average degree.
    \item For each motif, we also generated $S$ larger Erd\H{o}s-R\'{e}nyi random graphs, each of 100 vertices, again of the same average degree.
    \item We obtained positive/negative samples by connecting each larger Erd\H{o}s-R\'{e}nyi random graph with the motif/previously generated smaller Erd\H{o}s-R\'{e}nyi random graph (with certain edge probability).
    
\end{itemize}

In Table \ref{table:spantree_parameters_setting} we present additional details regarding architectures used in the experiments from Section \ref{sec:induced_cycles}, in particular the number of parameters and heads / polynomial filters used in different layers.
Ablation tests over GKAT random walk length for Section \ref{sec:induced_cycles} are presented in Fig. \ref{fig:spantree_GKAT_rwlen_ablation}.

\small
\begin{figure}[h]
\vspace{-3mm}
    \begin{minipage}{1.0\textwidth}
    \begin{center}
    \includegraphics[width=.50\linewidth]{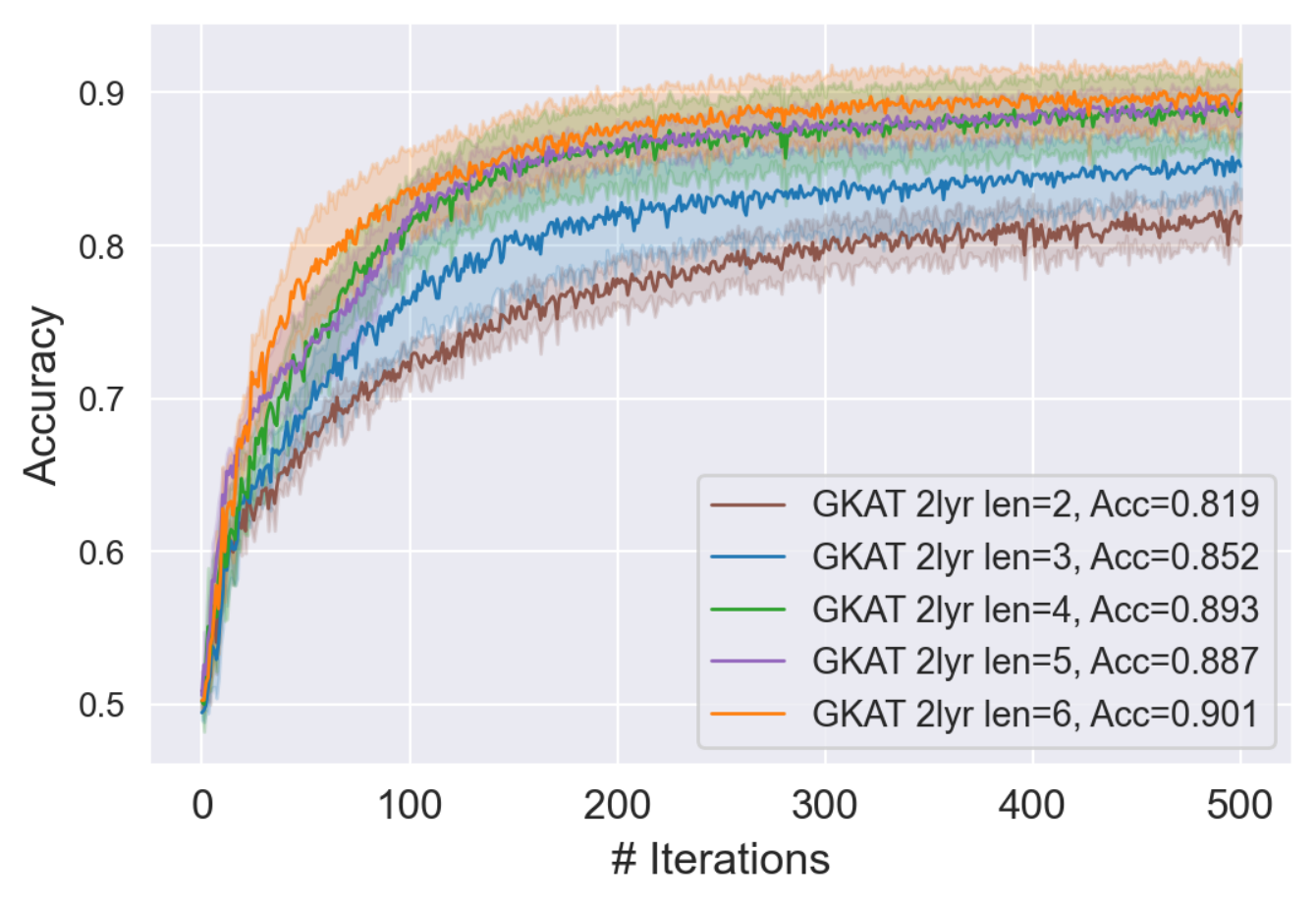}
    \end{center}
    \end{minipage}
    \vspace{-4.5mm}
\caption{\small{Ablation tests over random walk length of GKAT in Sec. \ref{sec:induced_cycles}.}}
\label{fig:spantree_GKAT_rwlen_ablation}
\end{figure}
\normalsize

\vspace{-1mm}
\begin{table}[h]
    \begin{center}
    \caption{\small{Additional details regarding architectures used in Section \ref{sec:induced_cycles}. For GKAT, we applied $8$ heads in the first layer, and $1$ head in the second layer, with $4$ hidden units in each attention head. The total number of trainable parameters was $242$. For GAT, we tested the number of layers from $2$ to $6$, changing the number of attention heads in each layer, but with the same number of hidden units in each attention head. For GCN, we modified the number of hidden units in each layer. For SGC, we modified the number of polynomial filters and the number of hidden units in each layer. The number of attention heads in GAT, as well as the number of hidden units in each layer in GCN and SGC were chosen to make their total number of trainable parameters comparable with the corresponding number of GKAT.}}
    \label{table:spantree_parameters_setting}
    \scalebox{0.9}{
     \begin{tabular}{ c c c c c} 
    \toprule
      & & $\textbf{\textrm{\#Heads}}$ & $\textbf{\textrm{Dim. Head}}$ & $\textbf{\textrm{\#Parameters}}$  \\ [0.5ex] 
    \toprule
    \textbf{GKAT}
    &  $\textrm{\textbf{2 layers}}$ & $[8, 1]$  & $4$ & $242$\\
    \toprule
    \multirow{3}*{\textbf{GAT}} 
    &  $\textrm{\textbf{2 layers}}$ & $[8, 1]$     & $4$ & $242$ \\
     & $\textrm{\textbf{3 layers}}$ & $[4, 2, 1]$  & $4$ & $242$ \\
     & $\textrm{\textbf{4 layers}}$ & $[4, 2, 1, 1]$ & $4$ & $266$ \\
     & $\textrm{\textbf{5 layers}}$ & $ [2, 2, 2, 1, 1]$  & $4$ & $258$ \\
     & $\textrm{\textbf{6 layers}}$ & $[3, 2, 1, 1, 1, 1] $ & $4$ & $270$ \\
    \toprule
     &  &  & {$\textbf{\textrm{Dim. Layer}}$}  & $\textbf{\textrm{\#Parameters}}$  \\ [0.5ex] 
    \toprule
    \multirow{3}*{\textbf{GCN}} 
    &  $\textrm{\textbf{2 layers}}$ & & {$[14, 14]$} & $268$\\
     & $\textrm{\textbf{3 layers}}$ & & {$[10, 10, 10]$} & $262$\\
     & $\textrm{\textbf{4 layers}}$ & & {$[8, 8, 8, 8]$}  & $250$\\
     & $\textrm{\textbf{5 layers}}$ & & {$[10, 8, 6, 6, 6]$} & $260 $\\
     & $\textrm{\textbf{6 layers}}$ & & {$[10, 6, 6, 6, 6, 6]$}  & $268$\\
    \toprule
     &  & $\textbf{\textrm{\#Polynomial\ Filters}}$ & {$\textbf{\textrm{Dim. Layer}}$}  & $\textbf{\textrm{\#Parameters}}$  \\ [0.5ex] 
    \toprule
    \multirow{3}*{\textbf{SGC}} 
    &  $\textrm{\textbf{2 layers}}$ & {$[4, 2]$} & $[10, 8]$ & $236$\\
     & $\textrm{\textbf{3 layers}}$ & {$[4, 2, 2]$} & $[8, 6, 6]$ & $234$\\
     & $\textrm{\textbf{4 layers}}$ & {$[8, 2, 2, 2]$} & {$[8, 5, 4, 4]$}  & $247$\\
     & $\textrm{\textbf{5 layers}}$ & {$[8, 2, 2, 2, 2]$} & $[6, 5, 4, 4, 4]$ & $245$\\
     & $\textrm{\textbf{6 layers}}$ & {$[8, 2, 2, 2, 2, 2]$} & {$[6, 4, 4, 4, 4, 3]$}  & $249$\\
    \bottomrule
    \end{tabular}}
    \end{center}
\end{table}
\normalsize

\subsection{GNNs for Bioinformatics Tasks \& Social Networks Data: Additional Details}
\label{app:chem_social_exp_details}

\subsubsection{Datasets Descriptions}
\label{app:chem_social_dataset_descriptions}
Detailed profiles of the datasets used in the experiments from Sec. \ref{sec:faircomparison} are given in Table \ref{table:dataset_statistics}.

\vspace{-1mm}
\begin{table}[h]
    \begin{center}
    \caption{\small{Bioinformatics and Social Dataset descriptions. \#NODES, \#EDGES and $\mathrm{Diameter}$ columns contain values averaged over all graphs in a given dataset.}    \vspace{2mm}}
    \label{table:dataset_statistics}
    \scalebox{0.9}{
     \begin{tabular}{ c  c  c  c  c  c  c  c  c } 
    \toprule
     & & $\textbf{\textrm{\#Graphs}}$ & $\textbf{\textrm{\#Classes}}$ & $\textbf{\textrm{\#Nodes}}$ & $\textbf{\textrm{\#Edges}}$ & $\textbf{\textrm{Diameter}}$ & $\textbf{\textrm{\#Features}}$  \\ [0.5ex] 
    \midrule 
     \multirow{4}*{\rotatebox[origin=c]{90}{\textbf{BIOINF.}}}
     & $\textbf{\textrm{D\&D}}$ & $1178$ & $2$ & $284.32$ & $715.66$ & $19.90$ & $89$  \\
     & $\textbf{\textrm{ENZYMES}}$ &  $600$ &  $6$ &  $32.63$ &  $64.14$  & $10.86$ &  $3$ \\
     & $\textbf{\textrm{NCI1}}$ & $4110$ &  $2$ &  $29.87$ &  $32.30$ & $13.26$  &  $37$ \\
     & $\textbf{\textrm{PROTEINS}}$ & $1113$ & $2$ & $39.06$ & $72.82$ & $11.48$ & $3$ \\
    \toprule
     \multirow{5}*{\rotatebox[origin=c]{90}{\textbf{SOCIAL}}} 
     & $\textbf{\textrm{COLLAB}}$ & $5000$ & $3$ & $74.49$ & $2457.78$ & $1.86$ & $1$ \\
     & $\textbf{\textrm{IMDB-BINARY}}$ & $1000$ & $2$ & $19.77$ & $96.53$ & $1.86$ & $1$ \\
     & $\textbf{\textrm{IMDB-MULTI}}$ & $1500$ & $3$ & $13.00$ & $65.94$ & $1.47$ & $1$ \\
     & $\textbf{\textrm{REDDIT-BINARY}}$ & $2000$ & $2$ & $429.63$ & $497.75$ & $9.72$ & $1$ \\
     & $\textbf{\textrm{REDDIT-5K}}$ &  $4999$ & $5$ & $508.82$ & $594.87$ & $11.96$ & $1$ \\
    \bottomrule
     
    \end{tabular}}
    \end{center}
\end{table}
\normalsize

For each dataset, we chose graphs with the number of nodes close to the average number of nodes shown in Table \ref{table:dataset_statistics}. Examples of bioinformatics-graphs from these datasets are given in Fig. \ref{fig:RepresentChemical}. Examples of social network graphs from these datasets are given in Fig. \ref{fig:RepresentSocial}.

\small
\begin{figure}[h]
\vspace{-3mm}
    \begin{minipage}{1.0\textwidth}
    \begin{center}
    \includegraphics[width=.85\linewidth]{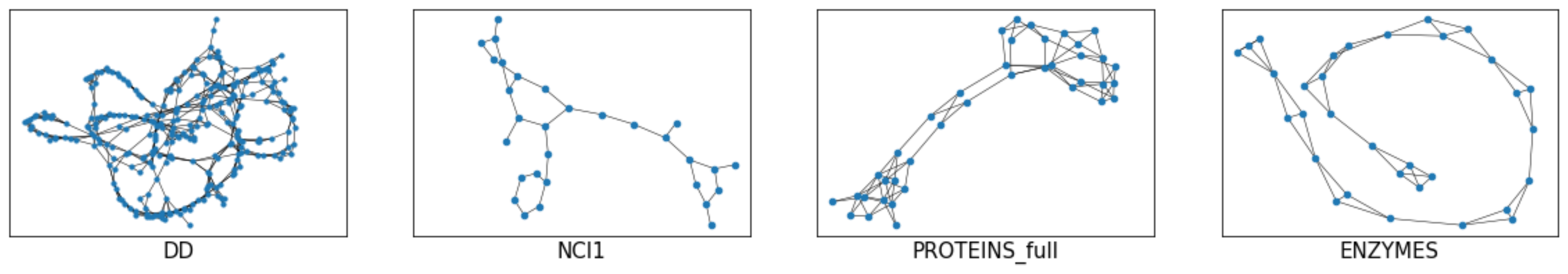}
    \end{center}
    \end{minipage}
    \vspace{-4.5mm}
\caption{\small{Representative plots for bioinformatics datasets. For each bioinformatics dataset, we chose the graph with number of nodes most similar to the average number of nodes shown in Table \ref{table:dataset_statistics}. }}
\label{fig:RepresentChemical}
\end{figure}
\normalsize

\small
\begin{figure}[h]
\vspace{3mm}
    \begin{minipage}{1.0\textwidth}
    \begin{center}
    \includegraphics[width=.95\linewidth]{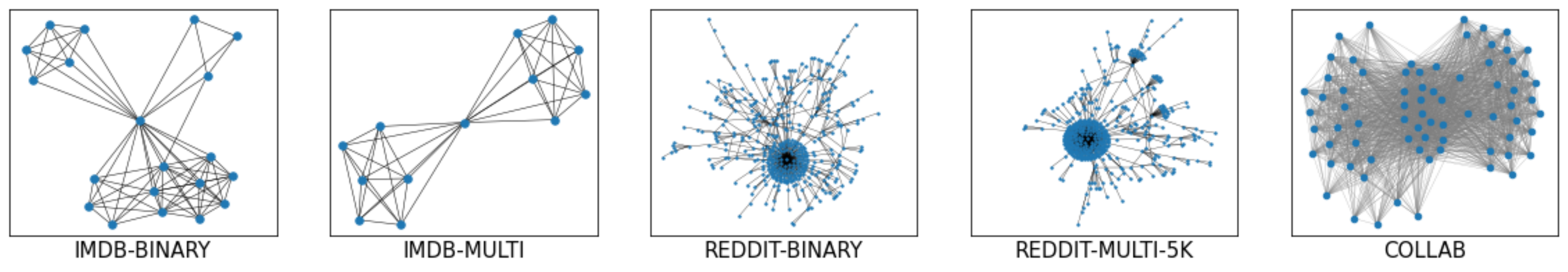}
    \end{center}
    \end{minipage}
    \vspace{-4.5mm}
\caption{\small{Representative plots for social datasets. For each social dataset, we chose the graph with number of nodes most similar to the average number of nodes shown in Table \ref{table:dataset_statistics}. }}
\label{fig:RepresentSocial}
\end{figure}
\normalsize

\subsubsection{Hyperparameter Selection}
\label{app:chem_social_hyper_parameter_settings}

In this section, we present details regarding  hyperparameter selection in Section \ref{sec:faircomparison} (see: Table \ref{table:real_dataset_parameters}). 

The tunable parameters included: general parameters like batch size, learning rate, dropout ratio, global pooling methods, regularization rate, data normalization methods, as well as parameters specific to our GKAT layers, which included: number of GKAT layers, number of attention heads and dimension of each head in a GKAT layer. We also tuned other options: whether to add a fully-connected layer after data normalization, but before GKAT layers, and dimension of fully-connected layers (both preceding and coming after GKAT layers). Due to the large amount of tunable parameters, we decided to first conduct a rough search for each parameter using only one random CV fold, select one/several parameter combination(s) with best performance, and then reused on all other folds.

For all other methods, we reported the best scores conducted via an extensive hyperparameters grid search \cite{FairComp}. For GKAT, we fixed the number of epochs to $E=1000$, early stopping patience as $500$ epochs, the criterion for early stopping as validation accuracy, global pooling method as summation, and used Adam optimizer. Then we performed hyperparameter tuning for: batch size $B \in \{32, 128\}$, learning rate $\eta \in \{0.01, 0.001, 0.0001\}$, dropout ratio $\in \{0.0, 0.1, 0.2, 0.4\}$, $L_{2}$-regularization rate $\in \{0.001, 0.005\}$, dimension of attention head in the first GKAT layer $h \in \{4, 8, 16, 32\}$, number of attention heads in the first GKAT layer $\in \{1, 4, 8, 12\}$, number of nodes in the MLP layer $\in \{32, 64, 128\}$, GKAT random walk length $\tau \in\{1,2,3,4\}$, whether to use a fully-connected layer before the first GKAT layer, and whether to apply batch normalization to pre-prosess data before feeding the model with it. 

For some of the datasets (e.g. D$\&$D), we selected the best hyperparameter set optimized over one random CV fold, and used it across all cross-validation outer folds.

\small
\vspace{-1mm}
\begin{center}
\begin{table}[h]
\caption{\small{Hyperparameter settings for the bioinformatics and social network datasets from Section \ref{sec:faircomparison}.} \vspace{1mm}}
\label{table:real_dataset_parameters}
\centering
\begin{tabular}{ c  c  c  c  c  c  c  c  c c c c} 
    \toprule
     & & \textbf{\textrm{BS}} &  \textbf{\textrm{\#Heads}} & \textbf{\textrm{d\textsubscript{Head}}} & \textbf{\textrm{d\textsubscript{FC}}} & \textbf{\textrm{Len\textsubscript{rw}}} & \textbf{\textrm{Drop}} & \textbf{\textrm{L2}} & \textbf{\textrm{add FC}} & \textbf{\textrm{Norm}}  \\ [0.5ex]
    \midrule 
     \multirow{4}*{\rotatebox[origin=c]{90}{\textbf{CHEM.}}}
     & \textbf{\textrm{D\&D}} & $32$ & $8$ & $16$ & $128$ & $2$ & $-$ & \textrm{0.005} &  \textrm{No} & \textrm{BN}  \\
     \cline{2-11}
     & \textbf{\textrm{NCI1}} & $32$ & $8$ & $32$ & $64$  & $4$ & $0.1$ & \textrm{0.001} &  \textrm{Yes} & \textrm{BN} \\
     \cline{2-11}
     & \multirow{2}*{\textbf{\textrm{PROTEINS}}} 
       & $32$  & $4$ & $8$  & $32$  & \multirow{2}*{$3$} & \multirow{2}*{$-$} & \multirow{2}*{$0.001$} &  \multirow{2}*{\textrm{Yes}} & \multirow{2}*{\textrm{BN}} \\
     & & $128$ & $8$ & $32$ & $128$ &            &                  & \\
     \cline{2-11}
     & \multirow{2}*{\textbf{\textrm{ENZYMES}}} 
     & \multirow{2}*{$32$} & $4$ & $16$ & $32$ & \multirow{2}*{$3$} & \multirow{2}*{$0.1$} &  \multirow{2}*{$0.001$}  & \multirow{2}*{\textrm{Yes}} & \multirow{2}*{\textrm{BN}} \\
     &          &          & $8$ & $32$ & $64$ &            &                  & \\
    \toprule
     \multirow{5}*{\rotatebox[origin=c]{90}{\textbf{SOCIAL}}} 
     & \textbf{\textrm{COLLAB}} & $32$ & $12$ & $4$ & $128$ & $2$ & $-$& $0.005$ & \textrm{No} & \textrm{No} \\
      \cline{2-11}
     & \multirow{2}*{\textbf{\textrm{IMDB-BINARY}}}
     & \multirow{2}*{$32$} & $8$ & $4$ & $64$ & \multirow{2}*{$1$} & \multirow{2}*{$-$}& \multirow{2}*{$0.005$} & \multirow{2}*{\textrm{No}} & \textrm{No} \\
     &          &          & $12$ & $8$ & $128$ &            &         &&         & \textrm{BN} \\
      \cline{2-11}
     & \multirow{2}*{\textbf{\textrm{IMDB-MULTI}}}
     & \multirow{2}*{$32$} & $8$ & $4$ & $64$ & \multirow{2}*{$1$} & \multirow{2}*{$-$} & \multirow{2}*{$0.005$} & \multirow{2}*{\textrm{No}} & \textrm{No} \\
     &          &          & $12$ & $8$ & $128$ &            &   & &               & \textrm{BN} \\
     \cline{2-11}
     & {\textbf{\textrm{REDDIT-BINARY}}} 
     & $32$ & {$4$} & $4$ & {$128$} & {$2$} & {$-$} & {$0.005$} & {\textrm{No}} & {\textrm{BN}}  \\
     \cline{2-11}
     & \textbf{\textrm{REDDIT-5K}} & $32$ & $8$ & $8$ & $64$ & $2$ & $-$& \textrm{0.005} & \textrm{No} & \textrm{BN}  \\
    \bottomrule
\end{tabular}
\label{tab:my_label}
\end{table} 
\end{center}
\vspace{-5mm}
\normalsize

\vspace{-10mm}
\subsection{Space \& Time Complexity Gains of GKAT: Additional Experiments}
\label{large-graphs}
Additionally, we have conducted experiments on much larger Erd\H{o}s-R\'{e}nyi graphs with motifs, see: Section \ref{sec:motif-detection}. The average number of nodes of each ER graph was 100K. Tested architectures had the same characteristics as in Section \ref{sec:motif-detection}. The results (final model accuracy) are presented in Table \ref{table:clocktime-large} for datasets: Caveman, Circle, Grid, Ladder and Circular-Ladder respectively.

\small
\begin{table}[h]
    \begin{center}
    \caption{\small{Running time of training different networks on datasets from Sec \ref{sec:motif-detection} but with much larger number of nodes ($\sim100\text{K}$). For GCN, GAT and SGC, we reported the accuracy with 2 layers. For GKAT, we used a 2-layer architecture and reported the accuracy with a fixed walk length of $3$ for motifs from Sec. \ref{sec:motif-detection}.}}
    {
     \begin{tabular}{ c  c  c  c  c  c  c} 
    \toprule
       & {$\textbf{\textrm{Caveman}}$} & {$\textbf{\textrm{Circle}}$} & {$\textbf{\textrm{Grid}}$}  & {$\textbf{\textrm{Ladder}}$}  & {$\textbf{\textrm{Circle Ladder}}$} \\ [0.5ex] 
    \toprule
     $\textbf{\textrm{GCN}}$ & $88.3\%$ & $82.7\%$ & $80.4\%$ & $80.6\%$ & $91.4\%$\\
     $\textbf{\textrm{GAT}}$ & $75.0\%$ & $81.0\%$ & $69.8\%$ & $77.0\%$ & $ 89.2\%$\\
     $\textbf{\textrm{SGC}}$ & $70.0\%$ & $80.4\%$ & $72.3\%$ & $76.1\%$ & $82.4\%$\\
    \bottomrule
     $\textbf{\textrm{GKAT}}$ & $\textbf{89.3\%}$ & $\textbf{83.2\%}$ & $\textbf{80.7\%}$ & $\textbf{81.5\%}$ & $\textbf{92.3\%}$ \\
    \bottomrule
    \end{tabular}}
\label{table:clocktime-large}    
\end{center}
\end{table}

\subsection{Experiments with Citation Networks Datasets}
\label{app:citation_exp_details}

\subsubsection{Datasets Descriptions}
\label{app:citation_dataset_descriptions}

\textbf{Datasets:} To directly compare GKAT with GAT, we also tested both algorithms on three publicly available citation networks datasets: Cora, Citeseer and Pubmed (\citep{sencollective}) with the same data splits as in \cite{velickovic}. Datasets descriptions are given in Table \ref{table:citation_datasets_description}.

\vspace{-1mm}
\begin{table}[h!]
    \begin{center}
    \caption{\small{Citation Networks Datasets Descriptions.}}
    \label{table:citation_datasets_description}
    \scalebox{0.9}{
     \begin{tabular}{ c c c c } 
    \toprule
     & $\textbf{\textrm{Cora}}$ & $\textbf{\textrm{Citeseer}}$ & $\textbf{\textrm{Pubmed}}$ \\ [0.5ex] 
    \toprule
     $\textbf{\textrm{\#Nodes}}$ & $2708$  & $3327$ & $19717$ \\
     $\textbf{\textrm{\#Edges}}$ & $5419$ & $4732$ & $44338$ \\
     $\textbf{\textrm{\#Features}}$ & $1433$ & $3703$ & $500$ \\
     $\textbf{\textrm{\#Classes}}$ & $7$ & $6$ & $3$ \\
     $\textbf{\textrm{\#Training Nodes}}$ & $140$ & $120$ & $60$ \\
     $\textbf{\textrm{\#Validation Nodes}}$ & $500$ & $500$ & $500$ \\
     $\textbf{\textrm{\#Test Nodes}}$ & $1000$ & $1000$ & $1000$ \\
    \bottomrule
    \end{tabular}}
    \end{center}
\end{table}
\normalsize

\subsubsection{Comparison with GAT}
\label{app:comparison_with_GAT}

\textbf{Experiment Settings:} We used the same model architecture and parameters as in GAT for our GKAT to make the comparison as accurate as possible. The only difference is that we replaced the adjacency matrix masking in GAT by the normalized dot-product based similarity matrix generated from random walks, as described in Section \ref{sec:gkat}. Both models used two-layer attention, with $8$ attention heads in the first layer, and $1$ head in the second layer. We used $8$ hidden units in the first layer, and the number of output units in the second layer was the same as number of classes. Each layer was followed by an exponential linear unit (ELU) activation. We applied $L_{2}$-regularization with $\lambda = 0.0005$, dropout with $p = 0.6$ for inputs and normalized attention coefficients in both layers for all three datasets.

\textbf{Results:} The results are shown in Table \ref{table:real_dataset_results_citation}. Our GKAT algorithm achieved lower accuracy on Cora dataset, but higher on the remaining two.

\begin{table}[h!]
    \begin{center}
    \caption{\small{Comparison of GAT and GKAT on citation networks datasets. For Cora and Citeseer, we reported the results for GAT from \citep{velickovic}. For GAT and Pubmed dataset, we reported the results averaged over 15 runs with the same parameter settings as in Cora and Citeseer. GKAT was run $15$ times over multiple random walk lengths up to $7$, and the best was reported. }}
    \label{table:real_dataset_results_citation}
    \scalebox{0.9}{
     \begin{tabular}{ c c c c} 
    \toprule
       & $\textbf{\textrm{Cora}}$ & $\textbf{\textrm{Citeseer}}$ & $\textbf{\textrm{Pubmed}}$  \\ [0.5ex] 
    \toprule
      $\textrm{\textbf{GAT}}$ & $\textbf{83.0\textpm0.7\%}$  & $72.5\pm0.7\%$ & $77.2\pm0.6\%$\\
    \toprule
      $\textrm{\textbf{GKAT}}$ & $82.1\pm0.7\%$  & $\textbf{73.0\textpm0.7\%}$ & $\textbf{78.0 \textpm 0.7\%}$ \\
    \bottomrule
    \end{tabular}}
    \end{center}
\end{table}
\normalsize

\textbf{Dynamic Generator of Random Walks: } 
We also tried the so-called \textit{dynamic}-$\mathrm{GKAT}$. The dynamic variant generated random walks from scratch in each training epoch, thus requiring additional compute. However one advantage of the dynamic version is that we could assign different transition probabilities for adjacent nodes (rather than sampling next point of the walk uniformly at random). The transition probability matrix can be a masked attention matrix and we only need its actions on the $L$-dimensional vectors to compute probability vectors in the visited nodes. Since the GKAT-masked attention can be interpreted as a kernelizable attention of a product-kernel (the product of the kernel between feature vectors in nodes and the nodes in the graph) and each factor-kernel admits on expectation a linearization, the product-kernel also does it via the mechanism of the Cartesian-product random features (see: \cite{daniely}). Thus this matrix admits on expectation a lower-rank decomposition and thus based on our analysis from Sec. \ref{sec:low-rank}, the actions of that matrix on the input vectors can be efficiently computed.

An intuition behind that particular variant is that we assign higher transition probabilities for neighbors with higher attention coefficients.
The dynamic variant enabled us to improve accuracy of GKAT on Citeseer to \textbf{73.3}\% (with reduced \textbf{0.6}\% standard deviation).

\subsubsection{Ablation Tests on Random Walk Length for GKAT}
\label{app:ablation_test_on_rw_length}

Figure \ref{fig:citation_pathlen} compares the effect of random walk path length of GKAT algorithms on training for Cora, Citeseer and Pubmed datasets. We run GKAT with multiple random walk lengths up to $7$. The results show that a small path length no longer than $4$ is enough for GKAT and dynamic-GKAT, which supports our claim that short walks are sufficient for GKAT.

\vspace{-1.0mm}
\small
\begin{figure}[h]
\vspace{-3mm}
    \begin{minipage}{1.0\textwidth}
    \begin{center}
    \includegraphics[width=.80\linewidth]{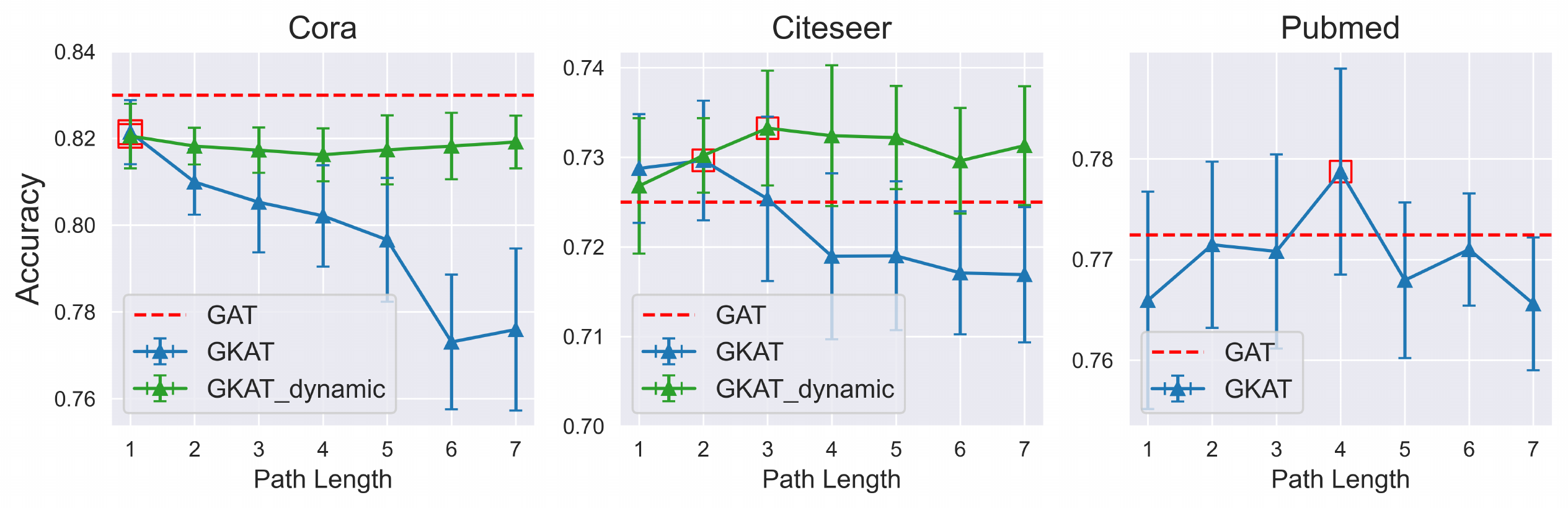}
    \end{center}
    \end{minipage}
    \vspace{-4.5mm}
    \caption{\small{Ablation tests over random walk path lengths of GKAT and dynamic-GKAT on Cora, Citeseer and Pubmed datasets. The errorbar represents 1 standard deviation. Test-accuracies for the optimized GAT were shown as horizontal red dotted lines in each subplot. Best test accuracies of GKAT algorithms were highlighted with red squares. The dynamic-GKAT was tested only on datasets, where accuracy advantage of the regular optimized GKAT over optimized GAT was $\leq 0.5\%$.}}
\label{fig:citation_pathlen}
\end{figure}
\normalsize

\subsection{Proof of Lemma \ref{first_mask_lemma} and Algorithm 1}

\begin{proof}
Note that the $ith$ token representation $\mathbf{r}_{i}$ obtained from the general masked kernel attention is of the following form:
\begin{equation}
\mathbf{r}_{i}=
\frac{\phi(\mathbf{q}_{i}^{\top})^{\top}\sum_{j=1}^{L}\mathbf{M}_{i,j}\phi(\mathbf{k}_{j}^{\top})\mathbf{v}_{j}}{\phi(\mathbf{q}_{i}^{\top})^{\top}\sum_{j=1}^{L}\mathbf{M}_{i,j}\phi(\mathbf{k}_{j}^{\top})},
\end{equation}
where $\mathbf{q}_{i},\mathbf{k}_{j},\mathbf{v}_{j}$ stand for the $ith$ query and $jth$ key/value row vectors respectively.
As in \citep{rpe-performers}, we define the following two sequences of matrices and vectors of shapes $\mathbf{R}^{m \times d}$ and $\mathbf{R}^{1 \times m}$ respectively:
\begin{align}
\begin{split}
\mathbf{D}^{1} = \left( \sum_{j=1}^{L}\mathbf{M}_{i,j}\phi(\mathbf{k}_{j}^{\top})\mathbf{v}_{j}\right)_{i=1}^{L}, \textrm{     } 
\mathbf{D}^{2} = \left( \sum_{j=1}^{L}\mathbf{M}_{i,j}\phi(\mathbf{k}_{j}^{\top})^{\top}\right)_{i=1}^{L}. 
\end{split}    
\end{align}
If we define $\tilde{\mathbf{D}}^{1}$ and $\tilde{\mathbf{D}}^{2}$ as the vectorized variants of $\mathbf{D}_{1}$ and $\mathbf{D}_{2}$, where each element of the sequence is (row) vectorized (note that the elements of $\mathbf{D}^{2}$ are already row-vectorized) and the resulting vectors are stacked into matrices, then: 
\begin{align}
\begin{split}
\tilde{\mathbf{D}}^{1} = \mathbf{M}\mathbf{V}^{1}, \textrm{  }
\tilde{\mathbf{D}}^{2} = \mathbf{M}\mathbf{V}^{2},
\end{split}    
\end{align}
with the $ith$ rows of $\mathbf{V}^{1}$ and $\mathbf{V}^{2}$ given as: 
$\mathbf{V}^{1}_{i}=\mathrm{vec}(\phi(\mathbf{k}_{i})^{\top}\mathbf{v}_{i})$, $\mathbf{V}^{2}_{i}=\phi(\mathbf{k}_{i}^{\top})^{\top}$.

We conclude that the computation of $\mathbf{D}_{1}$ and $\mathbf{D}_{2}$ takes time $T_{\mathbf{N}}(L)md$ and consequently all representations $\mathbf{r}_{i}$ can be computed in time $O((T_{\mathbf{N}}(L)+L)md)$.
That completes the proof of Lemma \ref{first_mask_lemma}.
\end{proof}
Note that Algorithm 1 follows immediately from the above proof. The algorithm consists of two phases. In the first phase, sequences $\mathbf{D}^{1}$ and $\mathbf{D}^{2}$ are computed (in the form of $\tilde{\mathbf{D}}^{1}$ and $\tilde{\mathbf{D}}^{2}$). In the second phase, they are used to get new token-representations (different tokens apply different elements of $\mathbf{D}^{1}$ and $\mathbf{D}^{2}$ for the computation of their new representations). 

\subsection{Proof of Lemma \ref{drpe-lemma}}
\begin{proof}
We define the so-called \textit{grid-ordering} on the nodes of the $d$-dimensional grid recursively.
For $d=1$ we take the natural ordering on the line.
For the $d$-dimensional grid with $d>1$, we take the $(d-1)$-dimensional slices: $\mathcal{S}_{1},\mathcal{S}_{2},...$ ordered according to their $d$th-dimension index. In each slice we recursively order all the tokens We then combine the orderings of all the slices to get the ordering on the $d$-dimensional grid.

We claim that the mask corresponding to the $d$-dimensional grid with the grid-ordering of tokens is $d$-level block-Toeplitz.

We will proceed by an induction on $d$. For $d=1$ the corresponding mask is of the form: $\mathbf{M}=[f(i-j)]_{i,j=1,...,L}$ for some learnable function $f$ and a natural ordering on the line. Therefore $\mathbf{M}$ is constant on each diagonal thus it is Toeplitz (e.g. 1-level block-Toeplitz). Now let us assume that $d>1$ and the result holds for $d-1$. We take the grid-ordering for the $d$-dimensional grid and the partitioning of the tokens given by the $(d-1)$-dimensional slices $\mathcal{S}_{1},\mathcal{S}_{2},...$ ordered according to their $d$th-dimension index. This partitioning induces the block-partitioning of the mask $\mathbf{M}$ (using grid-ordering) on the $d$-dimensional grid. Now note that the shortest-path distance between two nodes $v_{1}$ and $v_{2}$ from slices $\mathcal{S}_{i_{1}}$ and $\mathcal{S}_{i_{2}}$ respectively can be computed as: $\mathrm{dist}(v_{1},v_{2}^{\prime}) + |i_{1}-i_{2}|$, where $v_{2}^{\prime}$ is the projection of $v_{2}$ into slice $\mathcal{S}_{i_{1}}$. This observation combined with the inductive assumption implies that defined above block-partitioning of $\mathbf{M}$ produces matrices $\mathbf{B}^{i,j}$ from Definition \ref{def:block-toeplitz} and that completes the proof.
\end{proof}

\subsection{Proof of Lemma \ref{tree-lemma}}
\begin{proof}
Without loss of generality we can assume that $G_{\mathrm{base}}$ is a tree.
Assume that $\tau$ is of the form: $\tau(z) = az + b$ for some $a,b \in \mathbb{R}$.
Take some $\mathbf{x} \in \mathbb{R}^{L}$.
Let us root the tree $\mathcal{T}$ in one of its vertices that we denote as $v_{0}$. Denote by $v_{1},...,v_{k}$ for some $k \geq 0$ its neighbors. For every node $i$ we define $s_{i}$ as follows:
\begin{equation}
s_{i} = \sum_{j \in \mathcal{T}_{i}} \exp(\tau(\mathrm{dist}(i,j)))\mathbf{x}_{j},    
\end{equation}
where $\mathcal{T}_{i}$ denotes a subtree of $\mathcal{T}$ rooted in $i$. Our first observation is that all $s_{i}$ for $i=1,...,L$ can be computed in $O(L)$ time. To see this, note that:
\begin{align}
\begin{split}
s_{v_{0}} = \exp(\tau(\mathrm{dist}(v_{0},v_{0})))\mathbf{x}_{0} + \sum_{l=1,...,k} \sum_{j \in \mathcal{T}_{v_{l}}}
\exp(\tau(\mathrm{dist}(v_{0},j)))\mathbf{x}_{j} = \\
e^{b}\mathbf{x}_{0} + \sum_{l=1,...,k}\sum_{j \in \mathcal{T}_{v_{l}}}e^{a \cdot \mathrm{dist}(v_{0},j)+b}\mathbf{x}_{j} = 
e^{b}\mathbf{x}_{0} + \sum_{l=1,...,k}\sum_{j \in \mathcal{T}_{v_{l}}}e^{a \cdot (\mathrm{dist}(v_{l},j)+W(v_{0},v_{l}))+b}\mathbf{x}_{j} = \\
e^{b}\mathbf{x}_{0} + \sum_{l=1,...,k}e^{W(v_{0},v_{l})}\sum_{j \in \mathcal{T}_{v_{l}}}e^{a \cdot \mathrm{dist}(v_{l},j)+b}\mathbf{x}_{j}=
e^{b}\mathbf{x}_{0} + \sum_{l=1,...,k}e^{W(v_{0},v_{l})}s_{v_{l}}
\end{split}    
\end{align}
Thus we see that computing $s_{v_{0}}$ requires computing each $s_{v_{l}}$, followed by additional addition/multiplication operations that take time $O(\mathrm{deg}(v_{0}))$.
We conclude that we can recursively compute all $s_{i}$ in time $O(L)$.
Let us note that the entry of $\mathbf{w}=\mathbf{Mx}$ corresponding to node $i$ is of the form:
\begin{equation}
\mathbf{w}_{i} = \sum_{j \in \mathcal{T}} \exp(\tau(\mathrm{dist}(i,j)))\mathbf{x}_{j},
\end{equation}
Therefore ultimately we aim to compute all $\mathbf{w}_{i}$ for $i=1,...,L$ in time $O(L)$.
We observe that:
\begin{equation}
\mathbf{w}_{v_{0}} = s_{v_{0}}    
\end{equation}
Now take node $i \neq v_{0}$. Denote by $p(i)$ the predecessor of $i$ in $\mathcal{T}$. We have:
\begin{align}
\begin{split}
\label{r_eq}
\mathbf{w}_{i} = \sum_{j \in \mathcal{T}_{i}} \exp(\tau(\mathrm{dist}(i,j)))\mathbf{x}_{j}
+ \sum_{j \notin \mathcal{T}_{i}} \exp(\tau(\mathrm{dist}(i,j)))\mathbf{x}_{j} = \\
s_{i} + \sum_{j \notin \mathcal{T}_{i}}e^{a \cdot \mathrm{dist}(i,j)+b}\mathbf{x}_{j} = s_{i} + \sum_{j \notin \mathcal{T}_{i}} e^{a(W(i,p(i))+\mathrm{dist}(p(i),j))+b}\mathbf{x}_{j} = \\ s_{i} + e^{aW(i,p(i))}\sum_{j \notin \mathcal{T}_{i}}e^{a\mathrm{dist}(p(i),j)+b}\mathbf{x}_{j} = 
s_{i} + e^{aW(i,p(i))}t_{i},
\end{split}
\end{align}
where $t_{i} = \sum_{j \notin \mathcal{T}_{i}}e^{a\mathrm{dist}(p(i),j)+b}\mathbf{x}_{j}$.
Now note that:
\begin{equation}
\mathbf{w}_{p(i)} = t_{i} + e^{W(i,p(i))}s_{i}    
\end{equation}
and thus: $t_{i} = \mathbf{w}_{p(i)} - e^{W(i,p(i))}s_{i}$.
Plugging in the formula for $t_{i}$ into Equation \ref{r_eq}, we get:
\begin{equation}
\label{r-final-eq}
\mathbf{w}_{i} = e^{W(i,p(i))}\mathbf{w}_{p(i)} + (1-e^{2W(i,p(i))})s_{i}    
\end{equation}
We conclude that having computed all $s_{i}$ for $i=1,...,L$ in time $O(L)$, we can compute all $\mathbf{w}_{i}$ for $i=1,...,L$ in time $O(L)$ by ordering vertices in their increasing distance from the root $v_{0}$, setting up $\mathbf{w}_{v_{0}} = s_{v_{0}}$ and applying Equation \ref{r-final-eq}.
\end{proof}

\subsection{Proof of Theorem \ref{lapl-theory}}
\begin{proof}
We need the following definition.
\begin{definition}
A matrix $\mathbf{A}$ is Symmetric and Diagonally Dominant (SDD) if $\mathbf{A}_{i,j}=\mathbf{A}_{j,i}$ for all $i,j$ and $\mathbf{A}_{i,i} \geq \sum_{j \neq i}|\mathbf{A_{i,j}}|$.
\end{definition}

The results is a straightforward consequence of Theorem 1.2 from \cite{orecchia} and Lemma \ref{first_mask_lemma}. For Reader's convenience we restate that theorem here:
\begin{theorem}[SDD Matrix Exponential Computation]
Given an $L \times L$ SDD matrix $\mathbf{A}$, a vector $\mathbf{x}$ and a parameter $\delta \leq 1$, there is an algorithm that
computes a vector $\mathbf{u}$ such that $\|\exp(-\mathbf{A})\mathbf{x}-\mathbf{u}\| \leq \delta \|\mathbf{x}\|$ in time $\tilde{O}((|E|+L)\log(2+\|\mathbf{A}\|))$, Here tilde hides $\mathrm{poly}(\log(L))$ and $\mathrm{poly}(\mathrm{\log(\frac{1}{\delta})})$ factors.
\end{theorem}
It suffices to notice that both Laplacian matrix and its renormalized version are SDD. Furthermore, by Lemma \ref{first_mask_lemma}, fast (approximate) computation of $\exp(-\mathbf{\lambda A})\mathbf{x}$ for any $\mathbf{x} \in \mathbb{R}^{L}$ and $\mathbf{A}$ as in Theorem \ref{lapl-theory} leads to fast computation of the low-ranked attention with mask $\mathbf{M}=\exp(-\lambda\mathbf{A})$, as explained in Algorithm 1.
\end{proof}

\subsection{Proof of Theorem \ref{rwgnk_theorem}}
\label{sec:rwgnk_proof}
\begin{proof}
Note first that since $\omega(k)$ and $\omega(l)$ are chosen independently, we have:
\begin{equation}
\mathrm{K}_{p}^{\lambda,0}(k,l)=
\mathbb{E}_{\omega(k)}[f_{k}^{\omega(k),\lambda}] \cdot
(\mathbb{E}_{\omega(l)}[f_{l}^{\omega(l),\lambda}])^{\top} = 
\mathbb{E}_{\mathrm{\omega}(k),\omega(l)}[f_{k}^{\omega(k),\lambda}(f_{l}^{\omega(l),\lambda})^{\top}]
\end{equation}
Denote: $X = f_{k}^{\omega(k),\lambda}(f_{l}^{\omega(l),\lambda})^{\top}$. The key observation is that $X$ can be rewritten as:
\begin{equation}
X = \sum_{u \in \mathrm{V}(\mathrm{G})} \sum_{\substack{(j_{1}=k,...,j_{a+1}=u) =  \mathrm{pref}(\omega(k)),\\
(j^{\prime}_{1}=l,...,j^{\prime}_{b+1}=u) = \mathrm{pref}(\omega(l))}} \lambda^{a} \lambda^{b} = \sum_{(j_{1}=k,...,j_{a+b+1}=l) \in \Omega(k,l)} \lambda^{a+b}    
\end{equation}
where $\Omega(k,l)$ is the multi-set of walks from $k$ to $l$ that are built from some prefix of $\omega(k)$ concatenated with some prefix of $\omega(l)$.
Therefore we can write $X$ as:
\begin{equation}
X = \sum_{r \in R(k,l)}\sum_{i=0}^{\mathrm{len}(r)}\lambda^{\mathrm{len}(r)}
1[\mathcal{E}(r,i)],
\end{equation}
where $R(k,l)$ is the set of walks from $k$ to $l$, $\mathrm{len}(r)$ stands for the length (number of edges) of walk $r$ and $\mathcal{E}(r,i)$ is an event that first $i$ edges of the walk $r$ (counting from $k$) form the prefix sub-walk of $\omega(k)$ and the remaining ones form the prefix sub-walk of $\omega(l)$.
Therefore we have:
\begin{align}
\begin{split}
\mathrm{K}_{p}^{\lambda,0}(k,l)=
\mathbb{E}_{\omega(k),\omega(l)}
\left[
\sum_{r \in R(k,l)}\sum_{i=0}^{\mathrm{len}(r)}\lambda^{\mathrm{len}(r)}
1[\mathcal{E}(r,i)]
\right]
= \\ \sum_{r \in R(k,l)}\sum_{i=0}^{\mathrm{len}(r)}\lambda^{\mathrm{len}(r)} \mathbb{P}_{\omega(k),\omega(l)}[\mathcal{E}(r,i)]
= \sum_{r \in R(k,l)}\sum_{i=0}^{\mathrm{len}(r)}\lambda^{\mathrm{len}(r)} \prod_{j=0}^{i-1}\frac{1-p}{\mathrm{deg}(r^{j})}
\prod_{t=0}^{\mathrm{len}(r)-i-1}\frac{1-p}{\mathrm{deg}(r^{\mathrm{len}(r)-1-t})},
\end{split}
\end{align}
where $r^{y}$ stands for the $y^{th}$ vertex of the walk $r$ starting from $k$ and $\mathrm{deg}(v)$ denotes the degree of a vertex $v$.

Therefore we obtain:
\begin{equation}
\sum_{r \in R(k,l)}\sum_{i=0}^{\mathrm{len}(r)}\left(\frac{(1-p)\lambda}{d_{\mathrm{max}}}\right)^{\mathrm{len}(r)} \leq \mathrm{K}_{p}^{\lambda,0}(k,l) \leq   
\sum_{r \in R(k,l)}\sum_{i=0}^{\mathrm{len}(r)}\left(\frac{(1-p)\lambda}{d_{\mathrm{min}}}\right)^{\mathrm{len}(r)}
\end{equation}

We conclude that:
\begin{equation}
\sum_{i=0}^{\infty}r_{k,l}(i) \left(\frac{(1-p)\lambda}{d_{\mathrm{max}}}\right)^{i}(i+1) \leq \mathrm{K}_{p}^{\lambda,0}(k,l) \leq 
\sum_{i=0}^{\infty}r_{k,l}(i) \left(\frac{(1-p)\lambda}{d_{\mathrm{min}}}\right)^{i}(i+1)
\end{equation}

To complete the proof, it suffices to notice that matrix $\mathrm{Adj}^{i}(\mathrm{G})$ encodes the number of walks of length $i$ between pairs of vertices in $\mathrm{G}$. 
\end{proof}

\subsection{Extensions of the results from Section \ref{sec:trees}}

We will provide here the proofs of the results presented in Section \ref{sec:additional_results}. We first introduce additional concepts wee will leverage in the proofs.

\begin{definition}[balanced separators]
Take some function $\mathbf{w}:V(G) \rightarrow \mathbb{R}_{\geq 0}$ and some $\alpha>0$. We say that a subset $\mathcal{S} \subseteq V(G)$ is the $\alpha$-balanced separator with respect to $\mathbf{w}$, if the set of vertices $\mathcal{C}$ of every connected component of the subgraph graph $G_{|V(G) \backslash \mathcal{S}}$ of $G$, induced by $V(G) \backslash \mathcal{S}$, satisfies: $\mathbf{w}(\mathcal{C}) \leq \alpha \cdot \mathbf{w}(V(G))$, where $\mathbf{w}(\mathcal{X}) \overset{\mathrm{def}}{=} \sum_{x \in \mathcal{X}} \mathbf{w}(x)$.
\end{definition}

\begin{lemma}
\label{lemma:sep}
If $G$ is a tree then for an arbitrary function $\mathbf{w}:V(G) \rightarrow \mathbb{R}_{\geq 0}$ the $\frac{1}{2}$-balanced separator consisting of two adjacent vertices can be found in time $O(L)$. 
\end{lemma}

\begin{proof}
The proof is given in the proof of Lemma 7.19 in \cite{pilipczuk} and relies on the standard tree-search.
\end{proof}

\subsubsection{The proof of Lemma \ref{lemma:fft-tree}}

\begin{proof}
Take some vector $\mathbf{x} \in \mathbb{R}^{L}$. The goal is to compute $\mathbf{Mx}$ in time $O(L\log^{2}(L))$. 
For the node $i$ in a tree $\mathcal{T}$, denote: 
\begin{equation}
s_{i}=\sum_{j \in \mathcal{T}} f(\mathrm{dist}(i,j))\mathbf{x}_{j}
\end{equation}
Thus we want to compute all $s_{i}$ in time $O(L\log^{2}(L))$. 
If $|V(\mathcal{T})| \leq 2$ then all the calculations can be trivially done in $O(1)$ time, so we will assume now that $|V(\mathcal{T})| > 2$.
Take the $\frac{1}{2}$-balanced separator $\{a,b\}$ in $\mathcal{T}$ (with respect to the standard measure that counts the number of vertices) that exists and can be found in time $O(L)$ by Lemma \ref{lemma:sep}. Denote by $T_{a}$ the set of those trees in $\mathcal{T}_{|V(\mathcal{T}) \backslash \{a,b\}}$ that are are incident to $a$ in $\mathcal{T}$ and by $T_{b}$ the set of those trees in $\mathcal{T}_{|V(\mathcal{T}) \backslash \{a,b\}}$ that are are incident to $b$ in $\mathcal{T}$. Note that one of these sets might be potentially empty. Let us assume, without loss of generality that $T_{a}$ is not empty.
Denote by $V_{a}$ the union of the set of all the vertices of all the elements of $T_{a}$ and by $V_{b}$ the corresponding set for $T_{b}$.
If $\frac{1}{10} \leq |V_{a}| \leq \frac{9}{10}|V(\mathcal{T})|$, take:
$\mathcal{T}_{1}$ to be the subtree of $\mathcal{T}$ induced by $V_{a} \cup \{a\}$ and $\mathcal{T}_{2}$ to be the subtree of $\mathcal{T}$ induced by $V_{b} \cup \{a,b\}$. Otherwise take this $c \in \{a,b\}$ such that $|V_{c}| > \frac{9}{10}|V(\mathcal{T})|$. Denote:
$T_{c}=\{T^{1},...,T^{m}\}$. Note that $m>0$ ($T_{c}$ is not empty).
By the definition of the balanced separator, we have: $|V(T^{i})| \leq \frac{1}{2}|V(\mathcal{T})|$ for $i=1,...,m$. On the other hand: $|V(T^{1})|+...+|V(T^{m})| \geq \frac{9}{10}|V(\mathcal{T})|$.
Denote by $i^{*}$ the smallest $i \in \{1,...,m\}$ such that $|V(T^{1})|+...+|V(T^{i^{*}})| \geq \frac{9}{10}|V(\mathcal{T})|$. Note that $i^{*} > 1$.
We have:
\begin{align}
\begin{split}
\frac{2}{5}|V(\mathcal{T})| = \frac{9}{10}|V(\mathcal{T})| - \frac{1}{2} |V(\mathcal{T})| \leq |V(T^{1})| + ... + |V(T^{i^{*}})| -|V(T^{i^{*}})|  \\ = |V(T^{1})| + ... + |V(T^{i^{*}-1})| \leq \frac{9}{10}|V(\mathcal{T})|
\end{split}
\end{align}

Denote by $\mathcal{T}_{1}$ a subtree of $\mathcal{T}$ induced by $V(T^{1}) \cup ... \cup V(T^{i^{*}-1}) \cup \{c\}$ and by $\mathcal{T}_{2}$ a subtree of $\mathcal{T}$ induced by $V(\mathcal{T}) \backslash (V(T^{1}) \cup ... \cup V(T^{i^{*}-1}))$.
Note that in both cases we obtain two trees: $\mathcal{T}_{1}$ and $\mathcal{T}_{2}$ sharing a single vertex and such that: $V(\mathcal{T}_{1}) \cup V(\mathcal{T}_{2})=V(\mathcal{T})$.
Furthermore, we have:
\begin{equation}
\label{eq:rec}
|V(\mathcal{T}_{1})| = f|V(\mathcal{T})| + c_{1},
|V(\mathcal{T}_{2})| = (1-f)|V(\mathcal{T})| + c_{2},
\end{equation}
for $c_{1},c_{2} \in \{0,1\}$ and $\frac{2}{5} \leq f \leq \frac{9}{10}$.
Denote: $\{v\} = V(\mathcal{T}_{1}) \cap V(\mathcal{T}_{2})$.

Denote: $y^{1}_{i} = \sum_{j \in Z^{1}_{i}} \mathbf{x}_{j}$ and $y^{2}_{i} = \sum_{j \in Z^{2}_{i}} \mathbf{x}_{j}$ for $i=1,...,|V(\mathcal{T})|$, where $Z^{k}_{i}$ for $k \in \{1,2\}$ stands for the set of vertices in $\mathcal{T}_{k}$ with distance $i$ from $v$. Note that all $y^{k}_{i}$ can be trivially computed in time $O(|V(\mathcal{T})|)$.

To compute all $s_{i}$ for $i=1,...,|V(\mathcal{T})|$, we first compute recursively the following expressions:
\begin{equation}
s^{k}_{i} = \sum_{j \in \mathcal{T}_{k}} f(\mathrm{dist}(i,j))\mathbf{x}_{j}
\end{equation}
for $i \in V(\mathcal{T}_{k})$ and $k \in \{1,2\}$.
In order to compute expressions $s_{i}$, in addition to expressions $s^{k}_{i}$, we need to include the cross-term contributions (for pairs of vertices where one is from $\mathcal{T}_{1}$ and the other from $\mathcal{T}_{2}$). Note that this can be trivially done in time $O(V(\mathcal{T}))$ as long as we have computed the following two vectors: $\mathbf{H}\mathbf{y}^{1}$ and $\mathbf{H}\mathbf{y}^{2}$, where $\mathbf{y}^{k}=(y^{k}_{1},...,y^{k}_{|V(\mathcal{T})|})^{\top}$ for $k \in \{1,2\}$ and $\mathbf{H}$ is the Hankel matrix with the first row of the form: $(f(2),f(3),...,f(|V(\mathcal{T})|+1))$ and the last column of the form: $(f(|V(\mathcal{T})|+1),...,f(|V(\mathcal{T})|+|V(\mathcal{T})|))^{\top}$. This can be done in time $O(|V(\mathcal{T})|\log(|V(\mathcal{T}|)))$ with Fast Fourier Transform.
We conclude that our algorithm needs two recursive calls for subproblems of sizes which are constant fractions of $|V(\mathcal{T})|$ and given in Eq. \ref{eq:rec}, as well as additional computations conducted in time $O(|V(\mathcal{T})|\log(|V(\mathcal{T}|)))$. That leads to the total time complexity $O(|V(\mathcal{T})|\log^{2}(|V(\mathcal{T}|)))$ which completes the proof.
\end{proof}

\subsubsection{The proof of Lemma \ref{lemma:diam}}

\begin{proof}
Let us root $\mathcal{T}$ in a fixed vertex $v_{0}$.
We denote by $\mathcal{T}_{i}$ the subtree of $\mathcal{T}$ rooted in $i$.
For every node $i$, we maintain an array $g_{i}$ of length $\mathrm{diam}(\mathcal{T})+1$, where: $g_{i}[l] = \sum_{j \in \mathcal{T}_{i}:\mathrm{dist}(i,j)=l} \mathbf{x}_{j}$.
Computing $g_{i}$ for vertices $i$ which are the leaves of the tree $\mathcal{T}$ rooted in $v_{0}$ can be trivially done in time $O(1)$ per vertex. Now assume that $i$ is not a leaf and denote by: $q_{1},...,q_{k}$ (for some $k>0$) its children.
Note that:  $g_{i}$ can be computed as follows:
\[
g_{i}[l] = \left\{\begin{matrix} \sum_{p=1}^{k} g_{q_{p}}[l-1], \textrm{ if } l \geq 1
\\ \mathbf{x}_{i}, \textrm{ if } l=0 \end{matrix}\right.
\]

We also define an array $h_{i}$ for every node $i$ as follows:
\begin{equation}
h_{i}[l] = \sum_{j \in \mathcal{T}:\mathrm{dist}(i,j)=l} \mathbf{x}_{j}    
\end{equation}

For a given array $\mathbf{z}$, denote by $\mathrm{circ}(\mathbf{z})$ its \textit{circulant-shift} given as:
$\mathrm{circ}(\mathbf{z})[l]=\mathbf{z}[l-1]$ for $l>0$ and $\mathrm{circ}(\mathbf{z})[0]=0$.
Note that: $h_{v_{0}}=g_{v_{0}}$. Furthermore, for $i \neq v_{0}$, we can compute $h_{i}$ from $g_{i}$ and $h_{p(i)}$, where $p(i)$ stands for the parent of $i$ in $\mathcal{T}$ (rooted in $v_{0}$), as follows:

\begin{equation}
h_{i} = g_{i} + \mathrm{circ}(h_{p(i)}-\mathrm{circ}(g_{i})),    
\end{equation}
where addition and subtraction are dimension-wise. Thus having computed all $g_{i}$, we can compute all $h_{i}$ by proceeding from the root $v_{0}$ in the order induced by the distance from the root. 
We conclude that calculating $h_{i}$ for all vertices $i$ takes time $O(L \cdot \mathrm{diam(\mathcal{T})})$. Therefore, as in the case of our proof for the unweighted tree with $f$ given as the exponential mapping of the affine transform, effectively we perform in two stages - bottom-up to compute $g_{i}$-arrays and from the root to the leaves to compute arrays $h_{i}$ (with the use of already computed arrays $g_{i}$). 

Denote: $\mathbf{w}=\mathbf{Mx}$. Note that:
\begin{equation}
\mathbf{w}_{i}=\sum_{l=0}^{\mathrm{diam}(\mathcal{T})}f(l)h_{i}(l)    
\end{equation}

Thus computing all $\mathbf{w}_{i}$ can be also conducted in time $O(L \cdot \mathrm{diam}(\mathcal{T}))$ and that completes the proof.
\end{proof}

\end{document}